\documentclass{article} 
\usepackage{iclr2025_conference,times}

\usepackage{amsmath}
\usepackage{amsthm}
\usepackage{amsfonts}
\usepackage{bm}

\def\eqref#1{equation~\ref{#1}}

\def\1{\bm{1}}

\def\vzero{{\bm{0}}}

\def\vmu{{\bm{\mu}}}

\def\va{{\bm{a}}}

\def\vf{{\bm{f}}}

\def\vo{{\bm{o}}}

\def\vs{{\bm{s}}}

\def\vw{{\bm{w}}}
\def\vx{{\bm{x}}}

\def\hvs{{\hat{\bm{s}}}}
\def\vy{{\bm{y}}}
\def\vz{{\bm{z}}}

\def\vpi{{\bm{\pi}}}
\def\vsigma{{\bm{\sigma}}}
\def\vmu{{\bm{\mu}}}
\def\vtau{{\bm{\tau}}}

\DeclareMathAlphabet{\mathsfit}{\encodingdefault}{\sfdefault}{m}{sl}
\SetMathAlphabet{\mathsfit}{bold}{\encodingdefault}{\sfdefault}{bx}{n}

\def\gA{{\mathcal{A}}}

\def\gD{{\mathcal{D}}}

\def\gH{{\mathcal{H}}}

\def\gM{{\mathcal{M}}}
\def\gN{{\mathcal{N}}}

\def\gQ{{\mathcal{Q}}}
\def\gR{{\mathcal{R}}}
\def\gS{{\mathcal{S}}}

\def\gW{{\mathcal{W}}}

\def\gZ{{\mathcal{Z}}}

\def\sI{{\mathbb{I}}}

\def\sR{{\mathbb{R}}}

\newcommand*{\abs}[1]{| #1 |}

\NewDocumentCommand{\norm}{sm}{\IfBooleanTF{#1}{\|#2\|}{\left\| #2 \right\|}}
\DeclareMathOperator*{\defeq}{\smash{\overset{\mathrm{def}}{=}}}

\newcommand{\E}{\mathbb{E}}

\newcommand{\R}{\mathbb{R}}

\DeclareMathOperator*{\argmax}{arg\,max}

\newcommand{\algo}{\textsc{ActSafe}}
\newcommand{\sgym}{\textsc{Safety-Gym}}
\newcommand{\rwrl}{\textsc{RWRL}}

\definecolor{dark-blue}{rgb}{0,0,0.42}
\usepackage[hidelinks,colorlinks=true,allcolors=dark-blue]{hyperref}
\usepackage[capitalize,noabbrev]{cleveref}
\usepackage{url}
\usepackage{wrapfig}
\usepackage[pdftex]{graphicx}
\usepackage{caption}
\usepackage{subcaption}
\usepackage{booktabs}
\usepackage{tabularx}
\usepackage{multirow}
\usepackage{xspace}
\usepackage{nicefrac}
\usepackage[inline]{enumitem}
\usepackage{mathtools}
\usepackage{lipsum}
\usepackage{tikz}
\usepackage{parskip}
\usepackage{subcaption}
\usepackage{algorithm}
\usepackage{algpseudocode}
\usepackage{pifont}
\usepackage{svg}
\usepackage{microtype}

\theoremstyle{plain}
\newtheorem{theorem}{Theorem}[section]

\newtheorem{lemma}[theorem]{Lemma}
\newtheorem{corollary}[theorem]{Corollary}
\theoremstyle{definition}
\newtheorem{definition}[theorem]{Definition}
\newtheorem{assumption}[theorem]{Assumption}

\theoremstyle{remark}

\Crefname{assumption}{Assumption}{Assumptions} 
\crefname{assumption}{assumption}{assumptions} 
\algrenewcommand{\algorithmiccomment}[1]{\hfill\ding{228} #1}
\algblockdefx{MRepeat}{EndRepeat}{\textbf{repeat}}{}
\algnotext{EndRepeat}
\title{\algo{}: Active Exploration with Safety Constraints for Reinforcement Learning}

\author{Yarden As, Bhavya Sukhija \thanks{Equal contribution. Correspondence to: \href{mailto:yardas@ethz.ch}{\textcolor{black}{\texttt{yardas@ethz.ch}}}} \\
ETH Z{\"u}rich
\And
Lenart Treven \\
ETH Z{\"u}rich
\And
Carmelo Sferrazza \\
UC Berkeley
\AND
Stelian Coros  \\
ETH Z{\"u}rich
\And
Andreas Krause  \\
ETH Z{\"u}rich
}

\iclrfinalcopy 
\begin{document}

\maketitle

\begin{abstract}
\looseness=-1
Reinforcement learning (RL) is ubiquitous in the development of modern AI systems. However, state-of-the-art RL agents require extensive, and potentially unsafe, interactions with their environments to learn effectively. 
These limitations confine RL agents to simulated environments, hindering their ability to learn directly in real-world settings. In this work, we present \algo{}, a novel model-based RL algorithm for safe and efficient exploration. 
\algo{} learns a well-calibrated probabilistic model of the system and plans optimistically w.r.t.~the epistemic uncertainty about the unknown dynamics, while enforcing pessimism w.r.t.~the safety constraints. Under regularity assumptions on the constraints and dynamics, we show that \algo{} guarantees safety during learning while also obtaining a near-optimal policy in finite time. In addition, we propose a practical variant of \algo{} that builds on latest model-based RL advancements and enables safe exploration even in high-dimensional settings such as visual control. We empirically show that \algo{} obtains state-of-the-art performance in difficult exploration tasks on standard safe deep RL benchmarks while ensuring safety during learning.
\end{abstract}

\section{Introduction}
Reinforcement learning (RL) is a powerful paradigm for sequential decision-making under uncertainty, with many applications in games \citep{mnih2013playingatarideepreinforcement,Silver2016}, recommender systems \citep{maystre2024optimizingaudiorecommendationslongterm}, nuclear fusion control \citep{Degrave2022}, data-center cooling \citep{lazic2018data} and robotics \citep{Lee_2020,brohan2023rt1roboticstransformerrealworld,cheng2024express}. Despite the notable progress, its application without any use of simulators remains largely limited. This is primarily because, in many cases, RL methods require
massive amounts
of data for learning while also being inherently unsafe during exploration.

In many real-world settings, environments are complex and rarely align exactly with the assumptions made in simulators.
Learning directly in the real world allows RL systems to close the sim-to-real gap and continuously adapt to evolving environments and distribution shifts. However, to unlock these advantages, RL algorithms must be sample-efficient and ensure safety throughout the learning process to avoid costly failures or risks in high-stakes applications. For instance, agents learning driving policies in autonomous vehicles must prevent collisions with other cars or pedestrians, even when adapting to new driving environments. This challenge is known as \emph{safe exploration}, where the agent’s exploration is restricted by safety-critical, often unknown, \emph{constraints that must be satisfied throughout the learning process}.

\looseness-1Several works study safe exploration
and have demonstrated state-of-the-art performance in terms of both safety and sample efficiency for learning in the real world~\citep{pmlr-v37-sui15,racecarSafeopt,berkenkamp2021bayesian,CooperSafeOpt,sukhija2022scalable,widmer2023tuning}. These methods maintain a ``safe set'' of policies during learning, selecting policies from this set to safely explore and gradually expand it. Under common regularity assumptions about the constraints, these approaches guarantee safety throughout learning. However, explicitly maintaining and expanding a safe set, limits these methods to low-dimensional policies, such as PID controllers. This makes them difficult to scale to more complex tasks such as those considered in deep RL.

The goal of this work is to address this gap. To this end, we propose a scalable model-based RL algorithm -- \algo{} -- for efficient and safe exploration. 
Crucially, \algo{} learns an uncertainty-aware dynamics model, which it uses to implicitly define and expand the safe set of policies.
We theoretically show that \algo{} ensures safety throughout learning and converges to a near-optimal policy within a finite number of episodes. Moreover, \algo{} is practical and integrates seamlessly with state-of-the-art dynamics modeling techniques, for instance Dreamer~\citep{hafner2023mastering}, delivering strong empirical performance.
Thus, \algo{} advances the frontier of safe RL methods, both in theory and practice. Our main contributions are summarized below.
\paragraph{Contributions}
\begin{itemize}[leftmargin=0.5cm]
    \item We propose \algo{}, a novel model-based RL algorithm for safe exploration in continuous state-action spaces.
    \algo{} maintains a \emph{pessimistic} set of safe policies and \emph{optimistically} selects policies within this set that yield trajectories with the largest model epistemic uncertainty. 
    \item We show that when the dynamics lie in a reproducing kernel Hilbert space (RKHS), \algo{} guarantees safe exploration. In addition, we provide a sample-complexity bound for \algo{}, illustrating that \algo{} obtains $\epsilon$-optimal policies in a finite number of episodes. To the best of our knowledge, we are the first to show \emph{safety and finite sample complexity} for safe exploration in model-based RL with continuous state-action spaces.
    \item In our experiments, we demonstrate that \algo{}, when combined with a Gaussian process dynamics model, achieves efficient and safe exploration. Additionally, we show that \algo{} scales to high-dimensional environments of the \sgym{} and \rwrl{} benchmarks, excelling in challenging exploration tasks with visual control while also incurring significantly fewer constraint violations.
\end{itemize}

\section{Related Works}
\paragraph{Constrained Markov decision processes (CMDP) for safe RL}
\looseness=-1 Safety in reinforcement learning can be modeled in various ways \citep{JMLR:v16:garcia15a,brunke2022safe}. Constrained Markov decision processes (CMDPs) serve as a natural option for this purpose, as they can encode unsafe behaviors through constraints and enjoy many classical results from planning in MDPs~\citep{altman-constrainedMDP}.
Learning and planning in CMDPs have been extensively explored in the RL community, both theoretically and in practice. Notably, the works of \citet{efroni2020explorationexploitation,NEURIPS2022_14a5ebc9,ding2024convergencesamplecomplexitynatural,müller2024trulynoregretlearningconstrained} derive sample complexity bounds for CMDPs in discrete state-action spaces, whereas \citet{achiam2017constrained,tessler2018rewardconstrainedpolicyoptimization,pmlr-v119-stooke20a,xu2021crponewapproachsafe,liu2022constrained,as2022constrained,sootla2022saute,huang2024safedreamer} develop deep RL algorithms for CMDPs in continuous state-action spaces. However, all the aforementioned works relax the requirement of safe exploration and thus do not ensure the safety during learning. This is in contrast to this work, where we tackle the hard problem of safe exploration.

\paragraph{Provably safe exploration}
\citet{turchetta2016safe,Wachi_Sui_Yue_Ono_2018,wachi2020safereinforcementlearning} focus on safe exploration in CMDPs with \emph{discrete} state-action spaces and a constraint function that lies in an RKHS. \citet{zheng2020constrained} study sample complexity for safe exploration in discrete CMDPs. For continuous state-action spaces, \citet{berkenkamp2021bayesian} and extensions thereof \citep{baumann2021gosafe, sukhija2022scalable, hübotter2024informationbased}, leverage ideas from safe Bayesian optimization \citep{pmlr-v37-sui15} to directly optimize over the policy parameters in a model-free manner. The proposed algorithms guarantee safe exploration and finite sample complexity for learning an $\epsilon$-optimal solution. When evaluated on real-world systems, 
these methods exhibit remarkable sample efficiency while also being safe during learning \citep{ CooperSafeOpt,kirschner19a,widmer2023tuning}. However, these approaches are limited to simple low-dimensional policies, e.g., PID controllers, and are hard to scale to policies with more than few parameters.
In a similar spirit, \citet{berkenkamp2017safe} propose a model-based RL algorithm for safe learning, where safety is modeled in terms of Lyapunov stability. Even though the method enjoys similar theoretical guarantees as \citet{berkenkamp2021bayesian}, it assumes access to a generative simulator and thus cannot be applied to traditional online RL settings. In contrast,~\citet{koller2018learning,curi2022safe} propose more practical safe learning methods in combination with model-predictive control (MPC). While these methods guarantee safety during learning, they lack optimality guarantees and are computationally expensive to run in real-time.

A common aspect among most of the aforementioned methods is their use of an intrinsic objective, such as the model epistemic uncertainty, to guide and restrain exploration. Crucially, these methods maintain a safe set of policies which they gradually expand during learning by sampling policies that yield the highest intrinsic reward. In this work, we build on this key insight to propose a model-based RL algorithm for online learning that enjoys the same kind of guarantees while also being applicable in real-world settings such as deep RL.

\paragraph{Safe exploration with deep RL} 
\looseness=-1 A common approach to safe exploration is the use of safety filters~\citep{dalal2018safe,wabersich2021predictive,curi2022safe}, which modify the actions produced by an unsafe policy to meet safety constraints before they are executed on the real system. A key advantage of safety filters is that they can be easily added to any ``off-the-shelf'' unsafe RL algorithm. However, while safety is ensured, safety filters can lead to arbitrarily bad exploration and therefore lack guarantees for optimality. The works of \citet{srinivasan2020learning,thananjeyan2021recovery} rely on learning safety critics that certificate state-actions as safe, either for policy optimization or during online data collection. These works provide strong empirical results, including demonstrations of safe policies on real hardware. In addition, following this approach, \citet{bharadhwaj2020conservative} upper bound the the probability of making infeasible policy updates. Another line of work, relies on guaranteeing feasibility of policy optimization algorithms. Notably, \citet{chow2019lyapunov} use Lyapunov functions to guarantee feasibility of policy gradients iterates and derive their analysis on discrete state-action spaces. \citet{usmanova2024log} propose Log-Barriers SGD (LBSGD), an optimization algorithm that ensures feasibility of all its iterates with barrier functions, showcasing its application in navigation tasks with image observations. More recently, \citet{as2024safeexplorationusingbayesian, ni2024safeexplorationapproachconstrained} use LBSGD for safe learning with \emph{greedy} policy gradients, i.e., without considering intrinsic rewards to expand the safe set of policies. Crucially, this form of greedy policy search may result in sub-optimal policies, as described in \Cref{sec: alg} and empirically shown in \Cref{sec:experiments}.

\section{Problem Setting}
\label{sec:problem-setting}
We consider a discrete-time, episodic, constrained Markov decision process (CMDP), where the goal is to find a policy that not only maximizes the reward but also keeps the accumulated costs below a specified threshold, i.e., satisfies a safety constraint. This type of formulation is common in real-world scenarios, such as robot navigation. In this setting, the reward could represent the negative distance to a target destination, while the costs could represent penalties, such as a cost of~1 incurred for each collision with an obstacle. The CMDP formulation allows us to separate these two objectives, thus ensuring constraint satisfaction and safety, for an optimal policy.
In this setup, we consider dynamical systems with additive noise and bounded running rewards $r$ and costs $c$
\begin{align}
\begin{aligned}
        \vs_{t+1} = \vf^*(\vs_t, \va_t) + \vw_t \label{eq:dynamics}, \ &(\vs_t, \va_t) \in \gS \times \gA, \ \vs(0) = \vs_0 \\
    &r(\vs, \va) \in [0, R_{\max}] \qquad \text{(Running reward)}  \\
    &c(\vs, \va) \in [0, C_{\max}] \qquad \text{(Running cost)}.
\end{aligned}
\end{align}
Here $\vs_t \in \gS \subset\R^{d_\vs}$ is the state, $\va_t \in \gA \subset \R^{d_\va}$ the control input, and $\vw_t \in \gW \subseteq \R^{d_\vs}$ the process noise. The dynamics $\vf^*$ are unknown and without loss of generality, the reward $r$ and cost $c$ are assumed to be known.

\paragraph{Task} In this work, we study the following constrained RL problem \citep{altman-constrainedMDP}
\begin{equation}
\begin{split}
   \max_{\vpi \in \Pi} J_r(\vpi, \vf^*) & \coloneq \max_{\vpi \in \Pi} \E_{\vs_0, \vpi} \left[ \sum^{T-1}_{t=0} r(\vs_t, \va_t) \right] \ \text{s.t. }  J_c(\vpi, \vf^*) \coloneq \E_{\vs_0, \vpi} \left[ \sum^{T-1}_{t=0} c(\vs_t, \va_t) \right] \leq d; \\ 
&\vs_{t+1} = \vf^*(\vs_t, \vpi_t(\vs_t)) + \vw_t. 
\end{split}
\label{eq:average cost formulation}
\end{equation}
\looseness=-1

We study the episodic setting, with episodes $n=1, \ldots, N$. At the beginning of the episode $n$, we deploy a policy $\vpi_n = (\vpi_{n, 0}, \vpi_{n, 1}, \dots, \vpi_{n, T - 1})$, chosen from the policy space $\Pi$ for a horizon of $T$ on the system. Next, we obtain the trajectory $\vtau_n = (\vs_{n,0}, \ldots, \vs_{n, T})$, which we add to a dataset of transitions  $\gD_n = \{(\vz_{n, i } = (\vs_{n, i}, \vpi_{n, i}(\vs_{n, i})), \vy_{n, i} = \vs_{n, i + 1})_{0 \le i < T}\}$ and use the collected data to learn a  model of $\vf^*$.

\section{\algo{}: Active Exploration with Safety Constraints} \label{sec: alg}
\looseness=-1 A key challenge in learning with safety constraints is ensuring that these constraints are not violated during exploration. In the following, we introduce an idealized version of \algo{}, which guarantees safe exploration for dynamical systems with Gaussian process dynamics~\footnote{These guarantees can be extended to more general well-calibrated models as in~\cite{curi2020efficient}}.
Moreover, we also provide a bound on the sample complexity of \algo{} for learning an $\epsilon$-optimal policy.
To the best of our knowledge, this is the first model-based safe RL algorithm for continuous state-action spaces that provides guarantees for both safety and sample complexity. In Section~\ref{sec:practical-modifications}, we discuss a practical variant scaling to more complex domains. Our choice of a model-based approach is motivated by its superior empirical sample efficiency \citep{chua2018deep,as2022constrained} as well as our theoretical analysis.

\subsection{Assumptions}
Theoretically studying safe exploration without any assumptions on the underlying dynamical system is an ill-posed problem. In the following, we make some assumptions on the underlying problem that are common in the model-based RL \citep{curi2020efficient, kakade2020information} and safe RL \citep{berkenkamp2021bayesian, baumann2021gosafe} literature.

\begin{assumption}[Continuity of $\vf^*$ and $\vpi$]
\label{ass:lipschitz-continuity}
The dynamics model $\vf^*$ is $L_{\vf}$--Lipschitz, the cost $c$ is $L_c$--Lipschitz, and all $\vpi \in \Pi$ are continuous. 
\end{assumption}

\begin{assumption}[Process noise distribution]
\looseness=-1
The process noise is i.i.d. Gaussian with variance $\sigma^2$, i.e., $\vw_t \stackrel{\mathclap{i.i.d}}{\sim} \gN(\vzero, \sigma^2\sI)$.
\label{ass:noise-properties}
\end{assumption}

We focus on the setting where $\vw$ is homoscedastic for simplicity. However, our framework can also be applied to the more general heteroscedastic and sub-Gaussian case \citep{sukhija2024optimistic, hübotter2024informationbased}.
\begin{assumption}[Initial safe seed]
    We have access to an initial nonempty safe set $\gS_0$ of policies, i.e., 
    $\forall \vpi \in \gS_0: J_c(\vpi) \leq d$ and $ \gS_0 \neq \emptyset$. 
    \label{ass:safe-seed}
\end{assumption}

This assumption is crucial since without any prior knowledge about the system, ensuring safety is unrealistic. Therefore, $\gS_0$ allows us to start the learning process by selecting policies from this set.
In practice, this safe set could be obtained from a simulator or offline demonstration data.

In the following, we assume that at each step $n$ we learn a mean estimate $\vmu_n$ of $\vf^*$ and can quantify our uncertainty $\vsigma_n$ over the estimate. This allows us to learn an uncertainty-aware model of $\vf^*$, which is crucial for exploration and safety.
More formally, we learn a well-calibrated statistical model of $\vf^*$ as defined below.
\begin{definition}[Well-calibrated statistical model of $\vf^*$, \cite{rothfuss2023hallucinated}]
\label{definition:well-calibrated-model}
\looseness=-1
    Let $\gZ \defeq \gS \times \gA$.
    An all-time well-calibrated statistical model of the function $\vf^*$ is a sequence $\{\gQ_{n}(\delta)\}_{n \ge 0}$, where
    \begin{align*}
        \gQ_n(\delta) \defeq \left\{\vf: \gZ \to \sR^{d_s} \mid \forall \vz \in \gZ, \forall j \in \{1, \ldots, d_s\}: \abs{\mu_{n, j}(\vz) - f_j(\vz)} \leq \beta_n(\delta) \sigma_{n, j}(\vz)\right\},
    \end{align*}
    if, with probability at least $1-\delta$, we have $\vf^* \in \bigcap_{n \ge 0}\gQ_n(\delta)$.
    Here, $f_{j}$, $\mu_{n, j}$ and $\sigma_{n, j}$ denote the $j$-th element in the vector-valued functions $\vf$, $\vmu_n$ and $\vsigma_n$ respectively, and $\beta_n(\delta) \in \sR_{\geq 0}$ is sequence of scalar functions that depends on the confidence level $\delta \in (0, 1]$ and is monotonically increasing in $n$. 
\end{definition}
Next, we assume that $\vf^*$ resides in a Reproducing Kernel Hilbert Space (RKHS) of vector-valued functions and show that this is sufficient for us to obtain a well-calibrated model.
\begin{assumption}
\label{ass:rkhs-func}
We assume that the functions $f^*_j$, $j=\{1, \dots, d_s\}$ lie in a RKHS with kernel $k$ and have a bounded norm $B$, that is $\vf^* \in \gH^{d_s}_{k, B}$, with $\gH^{d_s}_{k, B} = \{\vf \mid \norm{f_j}_k \leq B, j=\{1, \dots, d_s\}\}$. Moreover, we assume that $k(\vz, \vz) \leq \sigma_{\max}$ for all $\vz \in \gZ$.
\end{assumption}

\Cref{ass:rkhs-func} allows us to model $\vf^*$ with GPs for which the mean and epistemic uncertainty (${\bm \mu}_n(\vz) = [\mu_{n,j} (\vz)]_{j\leq d_s}$, and $\vsigma_n(\vz) = [\sigma_{n,j} (\vz)]_{j\leq d_s}$) have an analytical formula (c.f., \cref{eq:GPposteriors} in \cref{sec:proofs}).

\begin{lemma}[Well calibrated confidence intervals for RKHS, \citet{rothfuss2023hallucinated}]
    Let $\vf^* \in \gH_{k,B}^{d_s}$.
Suppose ${\vmu}_n$ and $\vsigma_n$ are the posterior mean and variance of a GP with kernel $k$ after episode $n$.
There exists $\beta_n(\delta)$, for which the sequence $(\vmu_n, \vsigma_n, \beta_n(\delta))_{n \ge 0}$ represents a well-calibrated statistical model of $\vf^*$.
\label{lem:rkhs-confidence-interval}
\end{lemma}
\looseness=-1
In summary, \Cref{ass:rkhs-func} and \Cref{lem:rkhs-confidence-interval} show that in the RKHS setting, a GP is a well-calibrated model.
For more general models like Bayesian neural networks (BNNs), methods such as \cite{kuleshov2018accurate} can be used for calibration. Overall, our results can also be extended beyond the RKHS setting to other classes of well-calibrated models similar to~\cite{curi2020efficient}.

\subsection{\algo: Algorithmic Framework}
A crucial element of safe exploration algorithms is the exploration--expansion dilemma~\citep{hübotter2024informationbased}. In the following, we explain this in further detail, we then present a sketch of \algo{} and finally the formal algorithm. 

\paragraph{Safe set expansion}To ensure the safety of the agent during the initial phases of learning, \algo{} begins exploration by selecting policies from $\gS_0$ (\Cref{ass:safe-seed}). This reduces uncertainty $\vsigma_n$ about $\vf^*$ for policies in $\gS_0$, allowing us to infer the safety of policies beyond $\gS_0$ and \emph{expand} the safe set (see \Cref{fig:expansion}).
Safe set expansion is critical in safe RL because the optimal policy may lie outside the initial safe set, and expanding the safe set is necessary to reach it. Unlike traditional RL, where exploration focuses on maximizing reward, safe RL methods must also explore to expand the safe set. Methods like optimism and Thompson sampling, which focus on reward maximization, do not address this need for safe set expansion \citep{pmlr-v37-sui15}.

\paragraph{Algorithm Sketch}
\algo{} operates in two stages; \begin{enumerate*}[label=\textbf{(\roman*)}]
    \item expansion by intrinsic exploration and
    \item exploitation of extrinsic reward.
\end{enumerate*} In the first stage, \algo{} uses the model epistemic uncertainty as an intrinsic reward $r^{\text{\text{explore}}}(\vs, \va) = \norm{\vsigma_{n-1}(\vs, \va)}$ and selects policies within the safe set that yield trajectories with high uncertainties. This enables \algo{} to efficiently reduce its uncertainty within the safe set and expand it. \algo{} performs the intrinsic exploration phase for a fixed number of episodes $n^*$ till the safe set is sufficiently large and then transitions to the second stage. In the exploitation stage, \algo{} greedily maximizes the extrinsic reward $r$, effectively aiming to solve the problem in \cref{eq:average cost formulation}. 

Most model-based safe RL methods~\citep{as2022constrained} focus only on the second stage and ignore safe set expansion. In contrast, the theoretically grounded approaches of~\citet{berkenkamp2021bayesian, baumann2021gosafe, sukhija2022scalable, hübotter2024informationbased} explicitly account for the expansion, but are not scalable to high dimensional policies typically considered in RL.

\begin{algorithm}[t]
    \caption{\textbf{\algo:}  \textsc{Active Exploration with Safety Constraints} (Expansion stage)}
        \label{alg:safe-exploration}
    \begin{algorithmic}[]
        \State {\textbf{Init:}}{ Aleatoric uncertainty $\sigma$, Probability $\delta$, Statistical model $(\vmu_0, \vsigma_0, \beta_0(\delta))$}
        \For{episode $n=1, \ldots, n^*$}             \State $\vpi_n = \argmax_{\vpi \in \gS_n} \max_{\vf \in \gM_n} \E_{\vtau^{\vpi, \vf}}\left[\sum_{t=0}^{T-1} \norm{\vsigma_{n-1}(\hvs_t, \vpi(\hvs_t))}\right]$ \algorithmiccomment{Prepare policy}
            \State $\gD_n \leftarrow \textsc{Rollout}(\vpi_n)$ \algorithmiccomment{Collect data}
            \State $\text{Update } (\gM_n, \gS_n) \leftarrow \gD_{1:n}$ \algorithmiccomment{Update statistical model and safe set}
        \EndFor
    \end{algorithmic}
\end{algorithm}

Next, we present our main algorithm. 
To ensure safety during learning, we maintain a conservative (pessimistic) estimate of the safe set which is defined below.
\begin{definition}
Let $\gM_n \defeq \gM_{n-1} \cap \gQ_{n}, \forall n \geq 1$  denote the set of plausible models, and  $P_{n}(\vpi) = \max_{\vf \in \gM_n} J_c(\vpi, \vf)$ our \emph{pessimistic} estimate of the expected costs w.r.t.~$\gM_n$. Then, we define the safe set $\gS_n$ as
\begin{equation}
\gS_n \defeq \gS_{n-1} \cup \left\{\vpi \in \Pi\setminus \gS_{n-1};  
\exists \vpi' \in \gS_{n-1} \ \text{s.t.} \ P_n(\vpi') + D(\vpi, \vpi') \leq d\right\}, \label{eq:safeset-eq}
\end{equation}
where
\begin{align*}
&D(\vpi, \vpi') =  \\
&\E_{\vtau^{\vpi'}}\left[\sum^{T-1}_{t=0} 
        \min\left\{L_c\norm{\vpi'(\vs_t) - \vpi(\vs_t)}, 2C_{\max}\right\} + TC_{\max} \min\left\{\frac{L_f\norm{\vpi'(\vs_t) - \vpi(\vs_t)}}{\sigma}, 1\right\}\right]
\end{align*}
\label{def:safeset}
\end{definition}

\paragraph{Interpretation of \cref{def:safeset}} 

We maintain a pessimistic estimate, $P_n$ of the constraint value function $J_c$ w.r.t.~our model set $\gM_n$. In \cref{eq:safeset-eq} we define the expansion operator for the safe set.
This operator adds new policies $\vpi$ that are not yet in the safe set, i.e., those in $\Pi\setminus \gS_{n-1}$, to $\gS_n$ if they are close to some policy $\vpi'$ from within the safe set.
The distance $D(\vpi, \vpi')$ measures how close the two policies are in terms of the underlying cost function, and it is similar to other distance metrics, such as the one in \citet[Theorem 2.1]{foster2024behavior}.

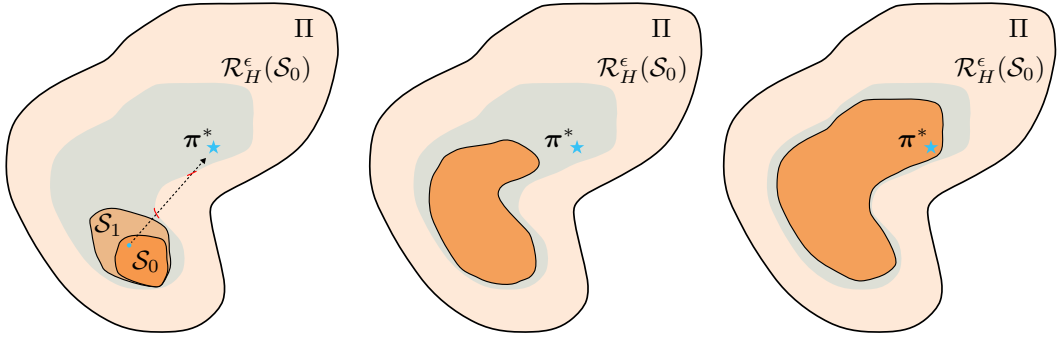
\begin{figure}
\vspace{5cm}
    \centering
\begin{tikzpicture}
    \begin{scope}[transform canvas={xshift=-7cm}]
        \includegraphics[width=\textwidth]{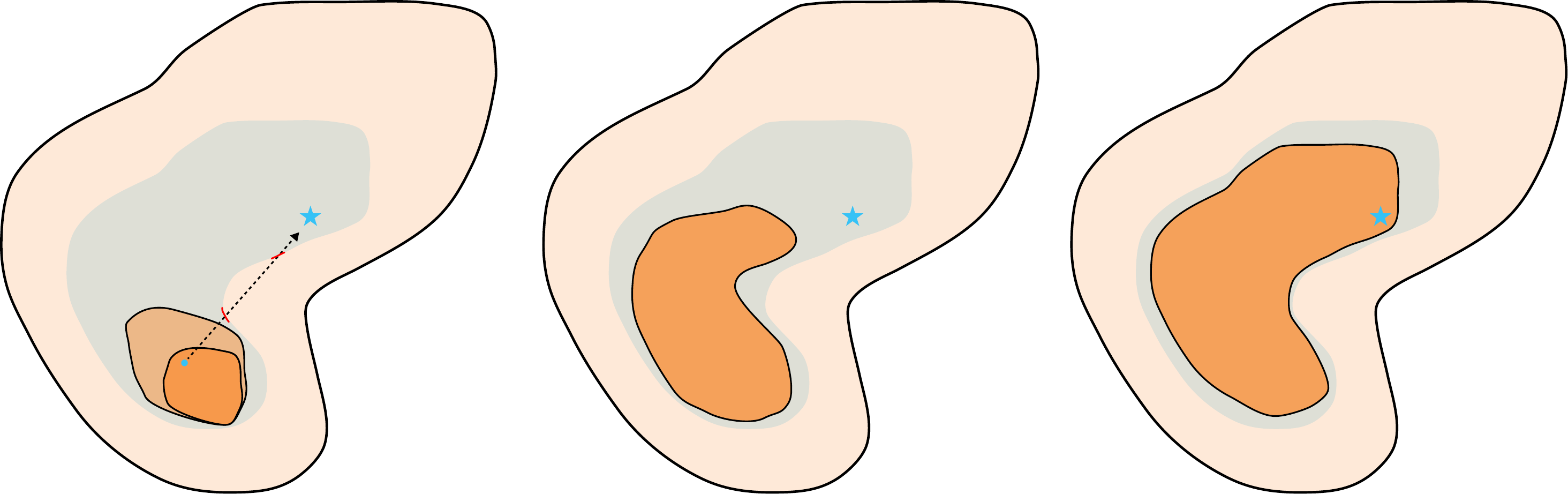}
        \node at (-0.5, 4.05) {$\Pi$};
        \node at (-5.25, 4.05) {$\Pi$};
        \node at (-10, 4.05) {$\Pi$};
        \node at (-10.5, 3.5) {$\gR^{\epsilon}_{H}(\gS_0)$};
        \node at (-5.5, 3.5) {$\gR^{\epsilon}_{H}(\gS_0)$};
        \node at (-0.75, 3.5) {$\gR^{\epsilon}_{H}(\gS_0)$};
        \node at (-12.1, 0.95) {$\gS_{0}$};
        \node at (-12.6, 1.43) {$\gS_{1}$};
        \node at (-11.4, 2.6) {$\vpi^*$};
        \node at (-6.6, 2.6) {$\vpi^*$};
        \node at (-1.9, 2.6) {$\vpi^*$};
    \end{scope}
\end{tikzpicture}
      \caption{
  \looseness=-1
  Schematic illustration of the expansion process. We expand the safe set at iteration $n - 1$ by reducing our uncertainty around policies at the boundary of $\gS_{n-1}$. The pale blue area depicts the reachable set $\gR_H^{\varepsilon}(\gS_0)$ after $H$ expansion iterations. The arrow on the leftmost illustration demonstrates that without explicit expansion, finding the optimal policy $\pi^*$ is intractable.}
  \vspace{-0.25cm}
  \label{fig:expansion}
\end{figure}

The expansion operator is common in the safe BO and RL literature \citep{racecarSafeopt,siemensSafeOpt,berkenkamp2021bayesian, baumann2021gosafe, 
CooperSafeOpt,sukhija2022scalable,
eventtriggeredbo,
fiedler2024safetysafebayesianoptimization}, and while it is generally difficult to evaluate in continuous spaces, it gives a key insight for safe RL methods: \emph{to effectively expand our knowledge of what is safe, we have to reduce our pessimism across policies in our safe set}.

\looseness-1Accordingly, during the expansion phase, we use the following objective for \algo{}, which, in the $n$-th episode, selects the policy $\vpi_n$ that yields the high uncertainty about the underlying dynamics 
\begin{align}
   \vpi_n,  \vf_n &= \argmax_{\vpi \in \gS_n, \vf \in \gM_n} \underbrace{\E_{\vtau^{\vpi, \vf}}\left[\sum_{t=0}^{T-1} \norm{\vsigma_{n-1}(\hvs_t, \vpi(\hvs_t))}\right]}_{\defeq J_{r_n}(\vpi, \vf)}.      
   \label{eq:exploration-op-optimistic}
\end{align}
Furthermore, akin to \citet{curi2020efficient, sukhija2024optimistic}, we introduce additional exploration, by also optimistically picking the dynamics $\vf_n$ from our set of plausible models $\gM_n$. Moreover, since the true dynamics, $\vf^*$ are unknown, we have to plan w.r.t.~some dynamics model in $\gM_n$. A theoretically grounded and well-established strategy for model-based RL methods is to pick an optimistic model $\vf_n$ from $\gM_n$. As we show in \cref{thm:main-theorem} this results in first-of-its-kind sample complexity and safety guarantees.
The expansion phase of the algorithm is summarized in \Cref{alg:safe-exploration}.

\begin{theorem}
Let \Cref{ass:lipschitz-continuity,ass:noise-properties,ass:rkhs-func,ass:rkhs-func,ass:safe-seed} hold. Then, we have with probability at least $1-\delta$ that $J_c(\vpi_n, \vf^*) \leq d$ $\forall n \geq 0$, i.e., \emph{\algo{} is safe during all episodes}. 

Moreover, consider any $\epsilon > 0$ and define $\gR^{\epsilon}_H(\gS_0)$ as the reachable safe set after $H$ expansions 
\begin{align*}
    \gR^{\epsilon}_H(\gS_0) &\defeq \gR^{\epsilon}_{H-1}(\gS_0) \cup \left\{\vpi\in \Pi\setminus \gR^{\epsilon}_{H-1}(\gS_0);  
\exists \vpi' \in \gR^{\epsilon}_{H-1}(\gS_0) \ \mathrm{s.t.} \ J_c(\vpi') + D(\vpi, \vpi') \leq d - \epsilon\right\} \\
\gR^{\epsilon}_0(\gS_0) &\defeq \gS_0.
\end{align*}
Let $n^*$ be the smallest integer such that
\begin{equation}
            \frac{n^*}{\gamma_{n^*}(k)                \beta^4_{n^*}(\delta)} \geq \frac{(H + 1) T^{6} C^4\frac{d_s \sigma^2_0}{\log(1 + \sigma^{-2}\sigma^2_0)}}{\epsilon^2}.
            \label{eq: bound on sample complexity},
        \end{equation}
where $C = (1 + \sqrt{d_s}) \max\{C_{\max}, R_{\max}, \sigma_0\}$, $\gamma_n(k)$  the maximum information gain~\citep{srinivas}, and $\Tilde{\vpi}_n$ the solution to $\underset{\vpi \in \gS_n}{\arg\max} \min_{\vf \in \gM_n} J_r(\vpi, \vf)$.
Then we have $\forall n \geq n^{*}$ with probability at least $1-\delta$
\begin{equation*}
    \max_{\vpi \in \gR^{\epsilon}_H(\gS_0)} J_r(\vpi) - J_r(\Tilde{\vpi}_n) \leq \epsilon.
\end{equation*}
\label{thm:main-theorem}
\end{theorem}
The theorem shows that \algo{} is safe during all episodes.
Furthermore, it shows that after finishing the expansion phase, \algo{} achieves an $\epsilon$-optimal solution within $\gR^{\epsilon}_H(\gS_0)$ for the underlying reward function $r$, where $\gR^{\epsilon}_H(\gS_0)$ is the largest safe set we can obtain after $H$ expansion steps if we know the dynamics to $\epsilon$ precision. To the best of our knowledge, we are the first to prove safety and give sample complexity bounds for safe model-based RL algorithms in the episodic setting with continuous state-action spaces.

Intuitively, by maximizing the epistemic uncertainty, we explore our dynamics uniformly among all policies in the safe set $\gS_n$, making our model more confident, i.e., reducing $\vsigma_n$. As our uncertainty within $\gS_n$ shrinks, we add more policies to our safe set (c.f.~\cref{def:safeset}) and thus facilitate its expansion. The proof of \cref{thm:main-theorem} is given in \cref{sec:proofs}. 

While the algorithm itself is difficult to implement for continuous state-action spaces, it gives key insights that guide our practical implementation in \Cref{sec:practical-modifications}: \begin{enumerate*}[label=\textbf{(\roman*)}]
    \item maximization of intrinsic rewards for expansion, 
    \item pessimism w.r.t.~plausible dynamics to define a safe set of policies $\gS_n$, and
    \item selecting $\vpi_n$ only from $\gS_n$ to ensure safety.
\end{enumerate*} Building on these insights, we introduce a practical version of \algo{} designed to excel in real-world scenarios, such as visual control tasks.

\subsection{Practical Implementation}
\label{sec:practical-modifications}
\paragraph{Optimizing over safe policies}
In \cref{eq:exploration-op-optimistic} we optimize over the policies within the safe set, where the safe set is defined according to~\cref{def:safeset}. 
This is particularly challenging in continuous state-action spaces since it requires us to maintain the model set $\gM_n$ and the safe set $\gS_n$.
We modify the definition of the safe set which makes the optimization problem more tractable.
\begin{equation}
    \widehat{\gS}_n = \left\{\vpi \in \Pi; \ \text{ s.t.} \max_{\vf' \in \gQ_n} J_c(\vpi, \vf') \leq d\right\}
\end{equation}
Note that $\widehat{\gS}_n \subseteq \gS_n$, making it a conservative estimate of $\gS_n$, therefore selecting policies from $\widehat{\gS}_n$ still preserves the safety guarantees. Furthermore, in $\widehat{\gS}_n$, we are pessimistic w.r.t.~the dynamics $\vf \in \gQ_n$ and thus we can simply use $\vmu_n, \vsigma_n$ to induce pessimism, i.e., we do not have to maintain the model set $\gM_n = \gM_{n-1} \cap \gQ_n$ (c.f.~\cref{definition:well-calibrated-model}). A similar relaxation is made by other safe RL algorithms such as~\cite{berkenkamp2021bayesian, baumann2021gosafe}. 

\looseness=-1
To practically solve \cref{eq:exploration-op-optimistic} we use $\widehat{\gS}_n$ instead of $\gS_{n}$, resulting in the following problem
\begin{align}
 \argmax_{\vpi \in \Pi} \max_{\vf \in \gQ_n} J_n(\vpi, \vf) \ \text{s.t.}\  \max_{\vf' \in \gQ_n} J_c(\vpi, \vf') \leq d.\label{eq:exploration-op-optimistic-practical}
\end{align}
\Cref{eq:exploration-op-optimistic-practical} is a constrained optimization problem with the added complexity of optimizing over the dynamics in $\gQ_n$. Moreover, it does not require us to maintain $\widehat{\gS}_n$ since we implicitly account for it in the constraint in \cref{eq:exploration-op-optimistic-practical}, making it tractable for continuous state-action spaces.
In \Cref{eq:exploration-op-optimistic-practical}, we have to solve $\max_{\vf' \in \gQ_n} J_c(\vpi, \vf')$ to enforce pessimism for safety. To this end, we use the methods from \citet{yu2020mopo} for our experiments. In practice, we solve \Cref{eq:exploration-op-optimistic-practical} by using a CMDP planner based on Log-Barrier SGD~\citep[LBSGD,][]{usmanova2024log}. Further technical details can be found in \Cref{sec:experiment-details}.

\begin{algorithm}[t]
    \caption{\textbf{\algo:} Practical Version}
       \label{alg:safe-practical}
    \begin{algorithmic}[]
        \State {\textbf{Init: }}{Model Set $\gQ_0$}
        \For{episode $n=1, \ldots, n^*$} \algorithmiccomment{Intrinsic exploration phase}
            \State Select $\vpi_n$ by solving \Cref{eq:exploration-op-optimistic-practical} \algorithmiccomment{Prepare policy}
            \State  $\gD_n \leftarrow \textsc{Rollout}(\vpi_n)$ \algorithmiccomment{Collect data}
             \State $\text{Update } \gQ_n \leftarrow \gD_{1:n}$ \algorithmiccomment{Update dynamics}
        \EndFor
        \For{episode $n = n^*, \ldots, N$} \algorithmiccomment{Extrinsic exploration phase}
            \State Select $\vpi_n$ by solving \Cref{eq:optimistic-extrinsic}
            \State  $\gD_n \leftarrow \textsc{Rollout}(\vpi_n)$
             \State $\text{Update } \gQ_n \leftarrow \gD_{1:n}$
         \EndFor
    \end{algorithmic}
\end{algorithm}

\paragraph{From CMDPs to visual control}
\looseness=-1
\algo{} can be seamlessly integrated with state-of-the-art model-based RL methods for learning in visual control tasks~\citep{hafner2019planet, hafner2023mastering}. To tighten the gap between RL and real-world problems, we relax the typical full observability assumption and consider problems where the agent receives an observation $\vo_t \sim p(\cdot | \vs_t)$ instead of $\vs_t$ at each time step. 
To handle partial observability, we choose to base our dynamics model on the Recurrent State Space Model (RSSM) introduced in \citet{hafner2019planet}. The RSSM can be thought of as a sequential variational auto-encoder that learns the (latent) dynamics $\vf$. We leverage approximate Bayesian inference techniques, in particular probabilistic ensembles~\citep{lakshminarayanan2017simple}, to approximate the posterior $p(\vf | \gD_n)$ over RSSMs. In particular, we learn an ensemble of $M$ models and define $\gQ_n$ as $\gQ_n = \bigcup^{M - 1}_{i=0} \{\vf^i\}$. The model's epistemic uncertainty (disagreement) is then used to enforce pessimism w.r.t. the safety constraints and for the intrinsic reward exploration  (see \cref{alg:safe-practical}).
\paragraph{Online policy improvement} 
After the first intrinsic exploration phase, it is often necessary to perform additional learning updates during the second exploitation phase~\citep{sekar2020planning}. Therefore, after $n^*$ iterations of intrinsic exploration, we optimize the extrinsic reward by solving
\begin{align}
 \argmax_{\vpi \in \Pi} \max_{\vf \in \gQ_n} J_r(\vpi, \vf) \ \text{s.t.}\  \max_{\vf' \in \gQ_n} J_c(\vpi, \vf') \leq d.\label{eq:optimistic-extrinsic}
\end{align}

\section{Experiments}
\label{sec:experiments}
In the following, we evaluate \algo{} on state-based and visual control tasks. For the state-based tasks, we use GPs to model the dynamics $\vf^*$. For the visual control tasks, we use the RSSM model from \citet{hafner2019planet} as described in \cref{sec:practical-modifications}. We thus validate both the theoretical and practical aspects of \algo{} in this section.

\subsection{Does \algo{} Explore Safely with GPs?}
\label{sec:gp-experiment}
 \begin{figure}
    \centering
    \includegraphics[width=0.95\textwidth]{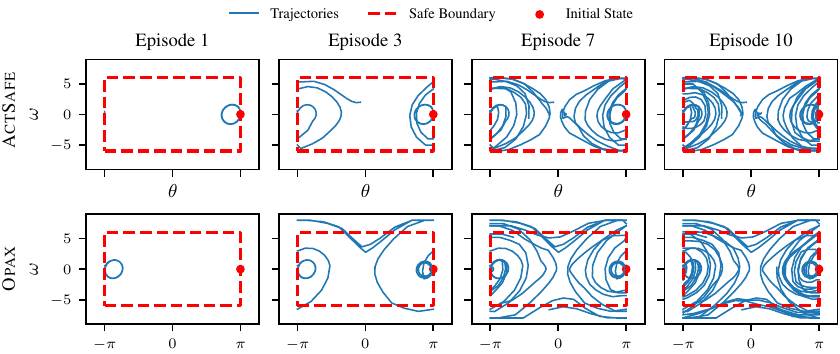}
    \caption{Safe exploration in the \textsc{PendulumSwingup} task. Each plot above visualizes trajectories considered during exploration across all past learning episodes. The red box in the plot depicts the safety boundary in the state space. \algo{} maintains safety throughout learning.}
    \vspace{-0.5cm}
    \label{fig:gp-experiment}
\end{figure}
\begin{wrapfigure}[15]{r}
{0.5\textwidth}
\vspace{-1.3\baselineskip}
  \begin{center}
        \includegraphics{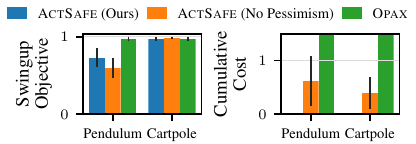}
  \end{center}
  \caption{Evaluation of safety via pessimism and intrinsic exploration. The cumulative cost accumulates all the incurred costs during learning, the reported objective performance is normalized. \algo{} maintains safety during learning while attaining high zero-shot performance on the \textsc{PendulumSwingup} objective at test time.}
  \label{fig:gp-summary}
\end{wrapfigure}
We evaluate \algo{} on the \textsc{Pendulum} and \textsc{Cartpole} environments. Additional details on the experimental setup, including the safety constraints, are provided in \cref{sec:experiment-details}. 
For both environments, we run the algorithms for ten episodes and then use the learned model to plan w.r.t.~known extrinsic rewards after the expansion phase. For extrinsic rewards, we consider the \textsc{Swingup} task.
We study the effects of pessimism with respect to the model uncertainty for safety. To this end, we consider as baselines a version of \algo{} without pessimism, which only uses the mean model $\vmu_n$ for planning and \textsc{Opax}~\citep{sukhija2024optimistic}, an unsafe active exploration algorithm.

We present our results in \Cref{fig:gp-summary}, where we report the performance and the total accumulated costs during exploration of our method. 
We conclude that \algo{} does not incur any costs during learning. In contrast, the variant of \algo{} without pessimism and \textsc{Opax} are unsafe during learning. This validates the necessity of using the model epistemic uncertainty to enforce pessimism during exploration. Note that \algo{} pays a price in terms of performance for pessimism, as it converges to a lower reward value than the other algorithms. 

 In \Cref{fig:gp-experiment} we visualize the trajectories of \algo{} and \textsc{Opax} in the state space during exploration. We observe that both algorithms cover the state space well, however, \algo{} remains within the safety boundary during learning whereas \textsc{Opax} violates the constraints.

\subsection{Does \algo{} Scale to Vision Control?}
\label{sec:vision-control-exp}
\looseness-1
While with GPs we can work closer to theory, scaling them to high dimensions with large data regimes, in particular visual control tasks, is challenging. We demonstrate our practical implementation (\Cref{alg:safe-practical}) on high-dimensional RL tasks. We highlight that ensuring safety with an NN model with randomly initialized weights is impractical without any additional prior knowledge. To this end, for all experiments hereon,  we assume access to an initial data collection (warm-up) period of 200K environment steps, where the agent collects data and uses it to calibrate its world model. This experimental setup is simple as it seamlessly integrates with both off and on-policy algorithms, such as CPO. Furthermore, it simulates a realistic setting, where the agent can collect some data initially in a controlled/supervised setting where safety is not directly penalized. However, after the initial data collection period, the agent is required to be safe during learning. We use the same training procedure across all baselines and environments ~\citep[akin to][]{dalal2018safe}.
In \Cref{sec:additional-experiments}, we present additional experiments that study safe exploration under distribution shifts of the dynamics, effectively leveraging the simulator 
to calibrate the model and imitating sim-to-real transfer.
\paragraph{Safety}
We investigate \algo{}'s performance in terms of constraint satisfaction during learning and compare it with state-of-the-art baseline algorithms for safe vision control ~\citep{as2022constrained,huang2024safedreamer} and with CPO~\citep{achiam2017constrained}. We use the same experimental setup from \sgym{} \citep{Ray2019} and \citet{as2022constrained}, with the \textsc{Point} robot in all tasks.
\begin{figure}
    \centering
\includegraphics[clip,width=0.95\textwidth]{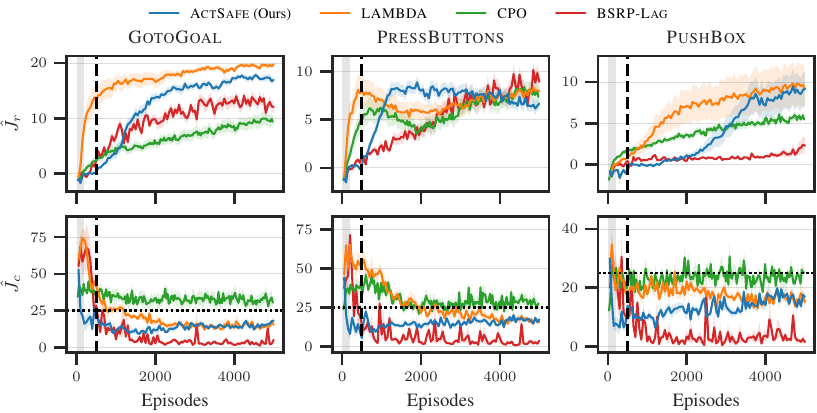}
    \caption{\looseness-1Safety of \algo{} in \sgym{} with vision control. The dotted horizontal line depicts the safety constraint. We report the mean and standard error across 10 seeds. The vertical dashed line illustrates the transition of \algo{} from the intrinsic exploraiton/expansion phase to the extrinsic reward phase. Grey shaded area represents the warm-up phase.}
    \vspace{-0.5cm}
    \label{fig:safety}
\end{figure}
As shown in \Cref{fig:safety}, compared to the baselines, \algo{}, significantly reduces constraint violation on all tasks. Notably, while \algo{} slightly  underperforms \textsc{LAMBDA}, it incurs much smaller costs. This result may be interpreted by the conservatism needed to maintain safety during learning. Furthermore, we observe that \textsc{BSRP-Lag} generally underperforms both algorithms in terms of safety and performance. We provide more details on our comparison in \Cref{sec:experiment-details}. Additionally, we ablate our choice of LBSGD in \Cref{sec:additional-experiments} and highlight its benefits.

\paragraph{Exploration}
In this experiment, we examine the influence of using an intrinsic reward in hard exploration tasks. To this end, we extend tasks from \sgym{} and introduce three new tasks with sparse rewards, i.e., without any reward shaping to guide the agent to the goal. We provide more details about the rewards in \Cref{sec:experiment-details}. In \Cref{fig:sparse-safety} we compare \algo{} with a \textsc{Greedy} baseline that collects trajectories only based on the sparse extrinsic reward. As shown, \algo{} substantially outperforms \textsc{Greedy} in all tasks, while violating the constraint only once in the \textsc{GotoGoal} task.
\begin{figure}
    \centering
\includegraphics[clip,width=0.9\textwidth]{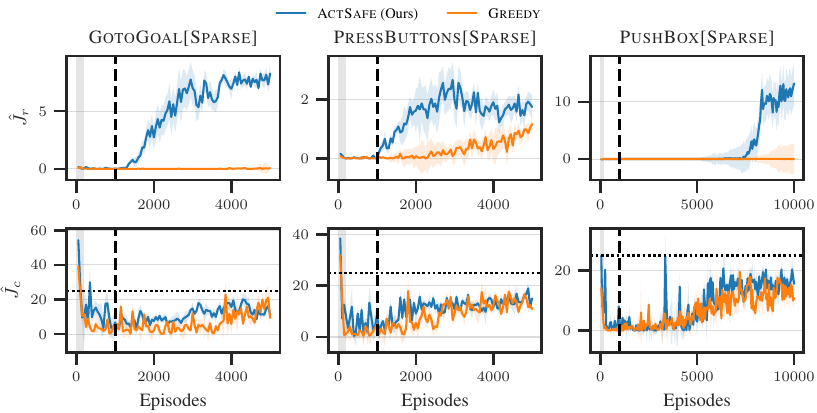}
    \caption{Performance on hard safe exploration tasks. The vertical dashed line illustrates the transition from data collection with intrinsic to extrinsic reward.}
    \vspace{-0.7cm}
    \label{fig:sparse-safety}
\end{figure}
In addition to \sgym{}, we evaluate on \textsc{CartpoleSwingupSparse} from \rwrl{}~\citep{dulacarnold2019challengesrealworldreinforcementlearning} with additional penalty for large actions \citep[see][and \Cref{sec:experiment-details}]{curi2020efficient}. We compare \algo{} with three baselines. \begin{enumerate*}[label=\textbf{(\roman*)}]
    \item \textsc{Uniform}, which samples actions uniformly at random during exploration, 
    \item \textsc{Optimistic}, which uses the model epistemic uncertainty estimates as exploration reward bonuses and
    \item \textsc{Greedy}, which optimizes the extrinsic reward directly.
\end{enumerate*}
\looseness=-1
\Cref{fig:cartpole-3e} indicates that using uncertainty quantification for exploration is crucial, as only \algo{} and \textsc{Optimistic} find non-trivial policies. Despite that, \algo{} outperforms \textsc{Optimistic}. Furthermore, even though \textsc{Uniform} initially explores with an unsafe policy, it is insufficient to learn a good dynamics model, and thus underperforms \algo{}. This is mainly due to the undirected exploration strategy of \textsc{Uniform}, which does not leverage the model's epistemic uncertainty.
\begin{wrapfigure}[13]{r}
{0.5\textwidth}
\vspace{-0.75\baselineskip}
  \begin{center}
        \includegraphics{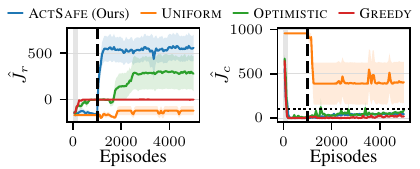}
  \end{center}
  \caption{Hard exploration performance in \textsc{CartpoleSwingupSparse} of \rwrl{} benchmark. We report the mean and standard error.}
  \label{fig:cartpole-3e}
\end{wrapfigure}

\vspace{-0.5cm}
\paragraph{Discussion}
\looseness -1
Our experiments underscore the following key findings. First, intuitively, in the GP setting, where our implementation is closer to theory, pessimism w.r.t.~the model uncertainty plays a crucial role as we achieve strict safe exploration. Second, in our visual control experiments, using a small fraction of data ($<$5$\%$ of total data collected) as the warm-up period for calibrating the model and policy is sufficient for drastically reducing constraint violation. Learning safely typically requires some form of prior knowledge about the problem, hence, using the data from the warm-up period keeps the experiment setup realistic 
without imposing specific domain knowledge and thus sacrificing generality. Third, 
in addition to exploring safely \algo{},  also solves tough exploration problems with the intrinsic rewards playing a crucial role.
These results underline the importance of intrinsic exploration in RL, especially in safety-critical tasks. Moreover, \algo{} transfers directly from the GP setting to the vision control setting and in both cases our results show that \algo{} outperforms the baselines in terms of both safety and performance. We provide additional experiments in \cref{sec:additional-experiments}, where ablate our choice of the LBSGD planner, evaluate \algo{} on a setting with distribution shifts in the dynamics and on a realistic robotics task from the state-of-the-art humanoid benchmark from \citet{sferrazza2024humanoidbench}.

\section{Conclusions}
In this paper, we introduce \algo{}, a safe model-based RL algorithm that leverages epistemic uncertainty as an intrinsic reward to learn a dynamics model efficiently and safely. We theoretically study systems with continuous state-action spaces and non-linear dynamics that lie in the RKHS, and provide guarantees on safety and near-optimality. We derive a practical variant of \algo{}, and demonstrate safe exploration and competitive performance with a Gaussian process dynamics model. Furthermore, we identify the key concepts that enable safe exploration with \algo{} and demonstrate how one can heuristically apply them to solve high-dimensional safe RL problems. Our empirical results showcase the importance of intrinsic rewards in the context of safety, demonstrating that \algo{} outperforms the baselines in the majority of tasks. In conclusion, \algo{} represents a significant advancement in safe reinforcement learning methods, enhancing both theoretical insights and practical applications.

\newpage

\subsubsection*{Acknowledgments}
We thank Jonas Hübotter for the insightful discussion and feedback on this work. 
This project has received funding from the Swiss National Science Foundation under NCCR Automation, grant agreement 51NF40 180545, the Microsoft Swiss Joint Research Center, grant of the Hasler foundation (grant no. 21039) and the SNSF Postdoc Mobility Fellowship 211086. Bhavya Sukhija was gratefully supported by ELSA (European Lighthouse on Secure and Safe AI) funded by the European Union under grant agreement No. 101070617.

\bibliography{main}

\begin{thebibliography}{69}
\providecommand{\natexlab}[1]{#1}
\providecommand{\url}[1]{\texttt{#1}}
\expandafter\ifx\csname urlstyle\endcsname\relax
  \providecommand{\doi}[1]{doi: #1}\else
  \providecommand{\doi}{doi: \begingroup \urlstyle{rm}\Url}\fi

\bibitem[Achiam et~al.(2017)Achiam, Held, Tamar, and Abbeel]{achiam2017constrained}
Joshua Achiam, David Held, Aviv Tamar, and Pieter Abbeel.
\newblock Constrained policy optimization.
\newblock In \emph{ICML}, 2017.

\bibitem[Altman(1999)]{altman-constrainedMDP}
E.~Altman.
\newblock \emph{Constrained Markov Decision Processes}.
\newblock Chapman and Hall, 1999.

\bibitem[As et~al.(2022)As, Usmanova, Curi, and Krause]{as2022constrained}
Yarden As, Ilnura Usmanova, Sebastian Curi, and Andreas Krause.
\newblock Constrained policy optimization via bayesian world models.
\newblock \emph{ICLR}, 2022.

\bibitem[As et~al.(2024)As, Sukhija, and Krause]{as2024safeexplorationusingbayesian}
Yarden As, Bhavya Sukhija, and Andreas Krause.
\newblock Safe exploration using bayesian world models and log-barrier optimization.
\newblock \emph{arXiv preprint arXiv:2405.05890}, 2024.

\bibitem[Baumann et~al.(2021)Baumann, Marco, Turchetta, and Trimpe]{baumann2021gosafe}
Dominik Baumann, Alonso Marco, Matteo Turchetta, and Sebastian Trimpe.
\newblock Gosafe: Globally optimal safe robot learning.
\newblock In \emph{ICRA}, 2021.

\bibitem[Berkenkamp et~al.(2017)Berkenkamp, Turchetta, Schoellig, and Krause]{berkenkamp2017safe}
Felix Berkenkamp, Matteo Turchetta, Angela Schoellig, and Andreas Krause.
\newblock Safe model-based reinforcement learning with stability guarantees.
\newblock \emph{NeurIPS}, 2017.

\bibitem[Berkenkamp et~al.(2021)Berkenkamp, Krause, and Schoellig]{berkenkamp2021bayesian}
Felix Berkenkamp, Andreas Krause, and Angela~P Schoellig.
\newblock Bayesian optimization with safety constraints: safe and automatic parameter tuning in robotics.
\newblock \emph{Machine Learning}, 2021.

\bibitem[Bharadhwaj et~al.(2020)Bharadhwaj, Kumar, Rhinehart, Levine, Shkurti, and Garg]{bharadhwaj2020conservative}
Homanga Bharadhwaj, Aviral Kumar, Nicholas Rhinehart, Sergey Levine, Florian Shkurti, and Animesh Garg.
\newblock Conservative safety critics for exploration.
\newblock \emph{arXiv preprint arXiv:2010.14497}, 2020.

\bibitem[Brohan et~al.(2022)Brohan, Brown, Carbajal, Chebotar, Dabis, Finn, Gopalakrishnan, Hausman, Herzog, Hsu, et~al.]{brohan2023rt1roboticstransformerrealworld}
Anthony Brohan, Noah Brown, Justice Carbajal, Yevgen Chebotar, Joseph Dabis, Chelsea Finn, Keerthana Gopalakrishnan, Karol Hausman, Alex Herzog, Jasmine Hsu, et~al.
\newblock Rt-1: Robotics transformer for real-world control at scale.
\newblock \emph{arXiv preprint arXiv:2212.06817}, 2022.

\bibitem[Brunke et~al.(2022)Brunke, Greeff, Hall, Yuan, Zhou, Panerati, and Schoellig]{brunke2022safe}
Lukas Brunke, Melissa Greeff, Adam~W Hall, Zhaocong Yuan, Siqi Zhou, Jacopo Panerati, and Angela~P Schoellig.
\newblock Safe learning in robotics: From learning-based control to safe reinforcement learning.
\newblock \emph{Annual Review of Control, Robotics, and Autonomous Systems}, 2022.

\bibitem[Cheng et~al.(2024)Cheng, Ji, Chen, Yang, Yang, and Wang]{cheng2024express}
Xuxin Cheng, Yandong Ji, Junming Chen, Ruihan Yang, Ge~Yang, and Xiaolong Wang.
\newblock Expressive whole-body control for humanoid robots.
\newblock \emph{arXiv preprint arXiv:2402.16796}, 2024.

\bibitem[Chow et~al.(2019)Chow, Nachum, Faust, Duenez-Guzman, and Ghavamzadeh]{chow2019lyapunov}
Yinlam Chow, Ofir Nachum, Aleksandra Faust, Edgar Duenez-Guzman, and Mohammad Ghavamzadeh.
\newblock Lyapunov-based safe policy optimization for continuous control.
\newblock \emph{arXiv preprint arXiv:1901.10031}, 2019.

\bibitem[Chua et~al.(2018)Chua, Calandra, McAllister, and Levine]{chua2018deep}
Kurtland Chua, Roberto Calandra, Rowan McAllister, and Sergey Levine.
\newblock Deep reinforcement learning in a handful of trials using probabilistic dynamics models.
\newblock \emph{NeurIPS}, 31, 2018.

\bibitem[Cooper \& Netoff(2022)Cooper and Netoff]{CooperSafeOpt}
Scott~E. Cooper and Th{\'e}oden~I. Netoff.
\newblock Multidimensional bayesian estimation for deep brain stimulation using the safeopt algorithm.
\newblock \emph{medRxiv}, 2022.

\bibitem[Curi et~al.(2020)Curi, Berkenkamp, and Krause]{curi2020efficient}
Sebastian Curi, Felix Berkenkamp, and Andreas Krause.
\newblock Efficient model-based reinforcement learning through optimistic policy search and planning.
\newblock \emph{NeurIPS}, 2020.

\bibitem[Curi et~al.(2022)Curi, Lederer, Hirche, and Krause]{curi2022safe}
Sebastian Curi, Armin Lederer, Sandra Hirche, and Andreas Krause.
\newblock Safe reinforcement learning via confidence-based filters.
\newblock In \emph{CDC}, 2022.

\bibitem[Dalal et~al.(2018)Dalal, Dvijotham, Vecerik, Hester, Paduraru, and Tassa]{dalal2018safe}
Gal Dalal, Krishnamurthy Dvijotham, Matej Vecerik, Todd Hester, Cosmin Paduraru, and Yuval Tassa.
\newblock Safe exploration in continuous action spaces.
\newblock \emph{arXiv preprint arXiv:1801.08757}, 2018.

\bibitem[Degrave et~al.(2022)Degrave, Felici, Buchli, Neunert, Tracey, Carpanese, Ewalds, Hafner, Abdolmaleki, de~las Casas, Donner, Fritz, Galperti, Huber, Keeling, Tsimpoukelli, Kay, Merle, Moret, Noury, Pesamosca, Pfau, Sauter, Sommariva, Coda, Duval, Fasoli, Kohli, Kavukcuoglu, Hassabis, and Riedmiller]{Degrave2022}
Jonas Degrave, Federico Felici, Jonas Buchli, Michael Neunert, Brendan Tracey, Francesco Carpanese, Timo Ewalds, Roland Hafner, Abbas Abdolmaleki, Diego de~las Casas, Craig Donner, Leslie Fritz, Cristian Galperti, Andrea Huber, James Keeling, Maria Tsimpoukelli, Jackie Kay, Antoine Merle, Jean-Marc Moret, Seb Noury, Federico Pesamosca, David Pfau, Olivier Sauter, Cristian Sommariva, Stefano Coda, Basil Duval, Ambrogio Fasoli, Pushmeet Kohli, Koray Kavukcuoglu, Demis Hassabis, and Martin Riedmiller.
\newblock Magnetic control of tokamak plasmas through deep reinforcement learning.
\newblock \emph{Nature}, 2022.

\bibitem[Ding et~al.(2022)Ding, Zhang, Duan, Ba{\c{s}}ar, and Jovanovi{\'c}]{ding2024convergencesamplecomplexitynatural}
Dongsheng Ding, Kaiqing Zhang, Jiali Duan, Tamer Ba{\c{s}}ar, and Mihailo~R Jovanovi{\'c}.
\newblock Convergence and sample complexity of natural policy gradient primal-dual methods for constrained mdps.
\newblock \emph{arXiv preprint arXiv:2206.02346}, 2022.

\bibitem[Dulac-Arnold et~al.(2019)Dulac-Arnold, Mankowitz, and Hester]{dulacarnold2019challengesrealworldreinforcementlearning}
Gabriel Dulac-Arnold, Daniel Mankowitz, and Todd Hester.
\newblock Challenges of real-world reinforcement learning.
\newblock \emph{arXiv preprint arXiv:1904.12901}, 2019.

\bibitem[Efroni et~al.(2020)Efroni, Mannor, and Pirotta]{efroni2020explorationexploitation}
Yonathan Efroni, Shie Mannor, and Matteo Pirotta.
\newblock Exploration-exploitation in constrained mdps.
\newblock \emph{arXiv preprint arXiv:2003.02189}, 2020.

\bibitem[Fiducioso et~al.(2019)Fiducioso, Curi, Schumacher, Gwerder, and Krause]{siemensSafeOpt}
Marcello Fiducioso, Sebastian Curi, Benedikt Schumacher, Markus Gwerder, and Andreas Krause.
\newblock Safe contextual {B}ayesian optimization for sustainable room temperature {PID} control tuning.
\newblock In \emph{IJCAI}, 2019.

\bibitem[Fiedler et~al.(2024)Fiedler, Menn, Kreisk{\"o}ther, and Trimpe]{fiedler2024safetysafebayesianoptimization}
Christian Fiedler, Johanna Menn, Lukas Kreisk{\"o}ther, and Sebastian Trimpe.
\newblock On safety in safe bayesian optimization.
\newblock \emph{arXiv preprint arXiv:2403.12948}, 2024.

\bibitem[Foster et~al.(2024)Foster, Block, and Misra]{foster2024behavior}
Dylan~J Foster, Adam Block, and Dipendra Misra.
\newblock Is behavior cloning all you need? understanding horizon in imitation learning.
\newblock \emph{arXiv preprint arXiv:2407.15007}, 2024.

\bibitem[Garc{{\'i}}a et~al.(2015)Garc{{\'i}}a, Fern, and o~Fern{{\'a}}ndez]{JMLR:v16:garcia15a}
Javier Garc{{\'i}}a, Fern, and o~Fern{{\'a}}ndez.
\newblock A comprehensive survey on safe reinforcement learning.
\newblock \emph{JMLR}, 2015.

\bibitem[Hafner et~al.(2019)Hafner, Lillicrap, Fischer, Villegas, Ha, Lee, and Davidson]{hafner2019planet}
Danijar Hafner, Timothy Lillicrap, Ian Fischer, Ruben Villegas, David Ha, Honglak Lee, and James Davidson.
\newblock Learning latent dynamics for planning from pixels.
\newblock In \emph{ICML}, 2019.

\bibitem[Hafner et~al.(2023)Hafner, Pasukonis, Ba, and Lillicrap]{hafner2023mastering}
Danijar Hafner, Jurgis Pasukonis, Jimmy Ba, and Timothy Lillicrap.
\newblock Mastering diverse domains through world models.
\newblock \emph{arXiv preprint arXiv:2301.04104}, 2023.

\bibitem[Holzapfel et~al.(2024)Holzapfel, Brunzema, and Trimpe]{eventtriggeredbo}
Antonia Holzapfel, Paul Brunzema, and Sebastian Trimpe.
\newblock Event-triggered safe {B}ayesian optimization on quadcopters.
\newblock In \emph{L4DC}, 2024.

\bibitem[Huang et~al.(2024)Huang, Ji, Xia, Zhang, and Yang]{huang2024safedreamer}
Weidong Huang, Jiaming Ji, Chunhe Xia, Borong Zhang, and Yaodong Yang.
\newblock Safedreamer: Safe reinforcement learning with world models.
\newblock In \emph{ICLR}, 2024.

\bibitem[H{\"u}botter et~al.(2024)H{\"u}botter, Sukhija, Treven, As, and Krause]{hübotter2024informationbased}
Jonas H{\"u}botter, Bhavya Sukhija, Lenart Treven, Yarden As, and Andreas Krause.
\newblock Information-based transductive active learning.
\newblock \emph{arXiv preprint arXiv:2402.15898}, 2024.

\bibitem[Kakade et~al.(2020)Kakade, Krishnamurthy, Lowrey, Ohnishi, and Sun]{kakade2020information}
Sham Kakade, Akshay Krishnamurthy, Kendall Lowrey, Motoya Ohnishi, and Wen Sun.
\newblock Information theoretic regret bounds for online nonlinear control.
\newblock \emph{NeurIPS}, 2020.

\bibitem[Kirschner et~al.(2019)Kirschner, Mutny, Hiller, Ischebeck, and Krause]{kirschner19a}
Johannes Kirschner, Mojmir Mutny, Nicole Hiller, Rasmus Ischebeck, and Andreas Krause.
\newblock Adaptive and safe {B}ayesian optimization in high dimensions via one-dimensional subspaces.
\newblock In \emph{ICML}, 2019.

\bibitem[Koller et~al.(2018)Koller, Berkenkamp, Turchetta, and Krause]{koller2018learning}
Torsten Koller, Felix Berkenkamp, Matteo Turchetta, and Andreas Krause.
\newblock Learning-based model predictive control for safe exploration.
\newblock In \emph{CDC}, 2018.

\bibitem[Kuleshov et~al.(2018)Kuleshov, Fenner, and Ermon]{kuleshov2018accurate}
Volodymyr Kuleshov, Nathan Fenner, and Stefano Ermon.
\newblock Accurate uncertainties for deep learning using calibrated regression.
\newblock In \emph{ICML}, 2018.

\bibitem[Lakshminarayanan et~al.(2017)Lakshminarayanan, Pritzel, and Blundell]{lakshminarayanan2017simple}
Balaji Lakshminarayanan, Alexander Pritzel, and Charles Blundell.
\newblock Simple and scalable predictive uncertainty estimation using deep ensembles.
\newblock \emph{NeurIPS}, 2017.

\bibitem[Lazic et~al.(2018)Lazic, Boutilier, Lu, Wong, Roy, Ryu, and Imwalle]{lazic2018data}
Nevena Lazic, Craig Boutilier, Tyler Lu, Eehern Wong, Binz Roy, MK~Ryu, and Greg Imwalle.
\newblock Data center cooling using model-predictive control.
\newblock \emph{NeurIPS}, 2018.

\bibitem[Lee et~al.(2020)Lee, Hwangbo, Wellhausen, Koltun, and Hutter]{Lee_2020}
Joonho Lee, Jemin Hwangbo, Lorenz Wellhausen, Vladlen Koltun, and Marco Hutter.
\newblock Learning quadrupedal locomotion over challenging terrain.
\newblock \emph{Science Robotics}, 2020.

\bibitem[Liu et~al.(2022)Liu, Cen, Isenbaev, Liu, Wu, Li, and Zhao]{liu2022constrained}
Zuxin Liu, Zhepeng Cen, Vladislav Isenbaev, Wei Liu, Steven Wu, Bo~Li, and Ding Zhao.
\newblock Constrained variational policy optimization for safe reinforcement learning.
\newblock In \emph{ICML}, 2022.

\bibitem[Maystre et~al.(2023)Maystre, Russo, and Zhao]{maystre2024optimizingaudiorecommendationslongterm}
Lucas Maystre, Daniel Russo, and Yu~Zhao.
\newblock Optimizing audio recommendations for the long-term: A reinforcement learning perspective.
\newblock \emph{arXiv preprint arXiv:2302.03561}, 2023.

\bibitem[Mnih et~al.(2015)Mnih, Kavukcuoglu, Silver, Rusu, Veness, Bellemare, Graves, Riedmiller, Fidjeland, Ostrovski, et~al.]{mnih2013playingatarideepreinforcement}
Volodymyr Mnih, Koray Kavukcuoglu, David Silver, Andrei~A Rusu, Joel Veness, Marc~G Bellemare, Alex Graves, Martin Riedmiller, Andreas~K Fidjeland, Georg Ostrovski, et~al.
\newblock Human-level control through deep reinforcement learning.
\newblock \emph{nature}, 2015.

\bibitem[M{\"u}ller et~al.(2024)M{\"u}ller, Alatur, Cevher, Ramponi, and He]{müller2024trulynoregretlearningconstrained}
Adrian M{\"u}ller, Pragnya Alatur, Volkan Cevher, Giorgia Ramponi, and Niao He.
\newblock Truly no-regret learning in constrained mdps.
\newblock \emph{arXiv preprint arXiv:2402.15776}, 2024.

\bibitem[Ni \& Kamgarpour(2024)Ni and Kamgarpour]{ni2024safeexplorationapproachconstrained}
Tingting Ni and Maryam Kamgarpour.
\newblock A safe exploration approach to constrained markov decision processes.
\newblock In \emph{ICML 2024 Workshop: Foundations of Reinforcement Learning and Control--Connections and Perspectives}, 2024.

\bibitem[Pinneri et~al.(2021)Pinneri, Sawant, Blaes, Achterhold, Stueckler, Rolinek, and Martius]{iCem}
Cristina Pinneri, Shambhuraj Sawant, Sebastian Blaes, Jan Achterhold, Joerg Stueckler, Michal Rolinek, and Georg Martius.
\newblock Sample-efficient cross-entropy method for real-time planning.
\newblock In \emph{CoRL}, 2021.

\bibitem[Ray et~al.(2019)Ray, Achiam, and Amodei]{Ray2019}
Alex Ray, Joshua Achiam, and Dario Amodei.
\newblock Benchmarking safe exploration in deep reinforcement learning.
\newblock \emph{arXiv preprint arXiv:1910.01708}, 2019.

\bibitem[Rothfuss et~al.(2023)Rothfuss, Sukhija, Birchler, Kassraie, and Krause]{rothfuss2023hallucinated}
Jonas Rothfuss, Bhavya Sukhija, Tobias Birchler, Parnian Kassraie, and Andreas Krause.
\newblock Hallucinated adversarial control for conservative offline policy evaluation.
\newblock \emph{UAI}, 2023.

\bibitem[Sekar et~al.(2020)Sekar, Rybkin, Daniilidis, Abbeel, Hafner, and Pathak]{sekar2020planning}
Ramanan Sekar, Oleh Rybkin, Kostas Daniilidis, Pieter Abbeel, Danijar Hafner, and Deepak Pathak.
\newblock Planning to explore via self-supervised world models.
\newblock In \emph{ICML}, 2020.

\bibitem[Sferrazza et~al.(2024)Sferrazza, Huang, Lin, Lee, and Abbeel]{sferrazza2024humanoidbench}
Carmelo Sferrazza, Dun-Ming Huang, Xingyu Lin, Youngwoon Lee, and Pieter Abbeel.
\newblock Humanoidbench: Simulated humanoid benchmark for whole-body locomotion and manipulation.
\newblock \emph{arXiv preprint arXiv:2403.10506}, 2024.

\bibitem[Silver et~al.(2016)Silver, Huang, Maddison, Guez, Sifre, van~den Driessche, Schrittwieser, Antonoglou, Panneershelvam, Lanctot, Dieleman, Grewe, Nham, Kalchbrenner, Sutskever, Lillicrap, Leach, Kavukcuoglu, Graepel, and Hassabis]{Silver2016}
David Silver, Aja Huang, Chris~J. Maddison, Arthur Guez, Laurent Sifre, George van~den Driessche, Julian Schrittwieser, Ioannis Antonoglou, Veda Panneershelvam, Marc Lanctot, Sander Dieleman, Dominik Grewe, John Nham, Nal Kalchbrenner, Ilya Sutskever, Timothy Lillicrap, Madeleine Leach, Koray Kavukcuoglu, Thore Graepel, and Demis Hassabis.
\newblock Mastering the game of go with deep neural networks and tree search.
\newblock \emph{Nature}, 2016.

\bibitem[Sootla et~al.(2022)Sootla, Cowen-Rivers, Jafferjee, Wang, Mguni, Wang, and Ammar]{sootla2022saute}
Aivar Sootla, Alexander~I Cowen-Rivers, Taher Jafferjee, Ziyan Wang, David~H Mguni, Jun Wang, and Haitham Ammar.
\newblock Saut{\'e} rl: Almost surely safe reinforcement learning using state augmentation.
\newblock In \emph{ICML}, 2022.

\bibitem[Srinivas et~al.(2012)Srinivas, Krause, Kakade, and Seeger]{srinivas}
Niranjan Srinivas, Andreas Krause, Sham~M. Kakade, and Matthias~W. Seeger.
\newblock Information-theoretic regret bounds for gaussian process optimization in the bandit setting.
\newblock \emph{IEEE Transactions on Information Theory}, 2012.

\bibitem[Srinivasan et~al.(2020)Srinivasan, Eysenbach, Ha, Tan, and Finn]{srinivasan2020learning}
Krishnan Srinivasan, Benjamin Eysenbach, Sehoon Ha, Jie Tan, and Chelsea Finn.
\newblock Learning to be safe: Deep rl with a safety critic.
\newblock \emph{arXiv preprint arXiv:2010.14603}, 2020.

\bibitem[Stooke et~al.(2020)Stooke, Achiam, and Abbeel]{pmlr-v119-stooke20a}
Adam Stooke, Joshua Achiam, and Pieter Abbeel.
\newblock Responsive safety in reinforcement learning by {PID} lagrangian methods.
\newblock In \emph{ICML}, 2020.

\bibitem[Sui et~al.(2015)Sui, Gotovos, Burdick, and Krause]{pmlr-v37-sui15}
Yanan Sui, Alkis Gotovos, Joel Burdick, and Andreas Krause.
\newblock Safe exploration for optimization with gaussian processes.
\newblock In \emph{ICML}, 2015.

\bibitem[Sukhija et~al.(2023)Sukhija, Turchetta, Lindner, Krause, Trimpe, and Baumann]{sukhija2022scalable}
Bhavya Sukhija, Matteo Turchetta, David Lindner, Andreas Krause, Sebastian Trimpe, and Dominik Baumann.
\newblock Gosafeopt: Scalable safe exploration for global optimization of dynamical systems.
\newblock \emph{Artificial Intelligence}, 2023.

\bibitem[Sukhija et~al.(2024)Sukhija, Treven, Sancaktar, Blaes, Coros, and Krause]{sukhija2024optimistic}
Bhavya Sukhija, Lenart Treven, Cansu Sancaktar, Sebastian Blaes, Stelian Coros, and Andreas Krause.
\newblock Optimistic active exploration of dynamical systems.
\newblock \emph{NeurIPS}, 2024.

\bibitem[Tassa et~al.(2018)Tassa, Doron, Muldal, Erez, Li, de~Las~Casas, Budden, Abdolmaleki, Merel, Lefrancq, Lillicrap, and Riedmiller]{tassa2018deepmind}
Yuval Tassa, Yotam Doron, Alistair Muldal, Tom Erez, Yazhe Li, Diego de~Las~Casas, David Budden, Abbas Abdolmaleki, Josh Merel, Andrew Lefrancq, Timothy Lillicrap, and Martin Riedmiller.
\newblock Deepmind control suite, 2018.

\bibitem[Tessler et~al.(2018)Tessler, Mankowitz, and Mannor]{tessler2018rewardconstrainedpolicyoptimization}
Chen Tessler, Daniel~J Mankowitz, and Shie Mannor.
\newblock Reward constrained policy optimization.
\newblock \emph{arXiv preprint arXiv:1805.11074}, 2018.

\bibitem[Thananjeyan et~al.(2021)Thananjeyan, Balakrishna, Nair, Luo, Srinivasan, Hwang, Gonzalez, Ibarz, Finn, and Goldberg]{thananjeyan2021recovery}
Brijen Thananjeyan, Ashwin Balakrishna, Suraj Nair, Michael Luo, Krishnan Srinivasan, Minho Hwang, Joseph~E Gonzalez, Julian Ibarz, Chelsea Finn, and Ken Goldberg.
\newblock Recovery rl: Safe reinforcement learning with learned recovery zones.
\newblock \emph{IEEE Robotics and Automation Letters}, 2021.

\bibitem[Turchetta et~al.(2016)Turchetta, Berkenkamp, and Krause]{turchetta2016safe}
Matteo Turchetta, Felix Berkenkamp, and Andreas Krause.
\newblock Safe exploration in finite markov decision processes with gaussian processes.
\newblock \emph{NeurIPS}, 29, 2016.

\bibitem[Usmanova et~al.(2024)Usmanova, As, Kamgarpour, and Krause]{usmanova2024log}
Ilnura Usmanova, Yarden As, Maryam Kamgarpour, and Andreas Krause.
\newblock Log barriers for safe black-box optimization with application to safe reinforcement learning.
\newblock \emph{JMLR}, 2024.

\bibitem[Vaswani et~al.(2022)Vaswani, Yang, and Szepesvari]{NEURIPS2022_14a5ebc9}
Sharan Vaswani, Lin Yang, and Csaba Szepesvari.
\newblock Near-optimal sample complexity bounds for constrained mdps.
\newblock In \emph{NeurIPS}, 2022.

\bibitem[Wabersich \& Zeilinger(2021)Wabersich and Zeilinger]{wabersich2021predictive}
Kim~Peter Wabersich and Melanie~N Zeilinger.
\newblock A predictive safety filter for learning-based control of constrained nonlinear dynamical systems.
\newblock \emph{Automatica}, 2021.

\bibitem[Wachi \& Sui(2020)Wachi and Sui]{wachi2020safereinforcementlearning}
Akifumi Wachi and Yanan Sui.
\newblock Safe reinforcement learning in constrained markov decision processes.
\newblock In \emph{ICML}, 2020.

\bibitem[Wachi et~al.(2018)Wachi, Sui, Yue, and Ono]{Wachi_Sui_Yue_Ono_2018}
Akifumi Wachi, Yanan Sui, Yisong Yue, and Masahiro Ono.
\newblock Safe exploration and optimization of constrained mdps using gaussian processes.
\newblock \emph{AAAI}, 2018.

\bibitem[Widmer et~al.(2023)Widmer, Kang, Sukhija, H{\"u}botter, Krause, and Coros]{widmer2023tuning}
Daniel Widmer, Dongho Kang, Bhavya Sukhija, Jonas H{\"u}botter, Andreas Krause, and Stelian Coros.
\newblock Tuning legged locomotion controllers via safe bayesian optimization.
\newblock In \emph{CoRL}, 2023.

\bibitem[Wischnewski et~al.(2019)Wischnewski, Betz, and Lohmann]{racecarSafeopt}
Alexander Wischnewski, Johannes Betz, and Boris Lohmann.
\newblock A model-free algorithm to safely approach the handling limit of an autonomous racecar.
\newblock In \emph{IEEE International Conference on Connected Vehicles and Expo}, 2019.

\bibitem[Xu et~al.(2021)Xu, Liang, and Lan]{xu2021crponewapproachsafe}
Tengyu Xu, Yingbin Liang, and Guanghui Lan.
\newblock Crpo: A new approach for safe reinforcement learning with convergence guarantee.
\newblock In \emph{ICML}, 2021.

\bibitem[Yu et~al.(2020)Yu, Thomas, Yu, Ermon, Zou, Levine, Finn, and Ma]{yu2020mopo}
Tianhe Yu, Garrett Thomas, Lantao Yu, Stefano Ermon, James~Y Zou, Sergey Levine, Chelsea Finn, and Tengyu Ma.
\newblock Mopo: Model-based offline policy optimization.
\newblock \emph{NeurIPS}, 2020.

\bibitem[Zheng \& Ratliff(2020)Zheng and Ratliff]{zheng2020constrained}
Liyuan Zheng and Lillian Ratliff.
\newblock Constrained upper confidence reinforcement learning.
\newblock In \emph{L4DC}, 2020.

\end{thebibliography}
\bibliographystyle{iclr2025_conference}

\appendix
\section{Proofs}
\label{sec:proofs}
In the following, we prove \cref{thm:main-theorem}. First, we provide the analytical formula for the mean and epistemic uncertainty of a GP model.
We denote $\vx \coloneq (\vs, \va)$, so that

\begin{align}
\begin{split}
    \mu_{n,j} (\vx)& = {\bm{k}}_{n}^\top(\vx)({\bm K}_{n} + \sigma^2_\epsilon \bm{I})^{-1}\vy_{n, j}  \label{eq:GPposteriors}\\
     \sigma^2_{n, j}(\vx) & =  k(\vx, \vx) - {\bm k}^T_{n}(\vx)({\bm K}_{n}+\sigma^2_\epsilon \bm{I})^{-1}{\bm k}_{n}(\vx)
\end{split}
\end{align}
where $\vy_{n, j} = [s'_{i, j}]^\top_{i \leq n}$ is the vector of the $j$-th element of the observed next states $\vs'_i$, $\bm{k}_{n}(\vx) = [ k(\vx, \vx_i)]^\top_{i \leq n}$, and ${\bm K}_{n} = [ k(\vx_i, \vx_l)]_{i,l \leq n}$ is the kernel matrix. By concatenating the element-wise posterior mean and standard deviation, we obtain $\vmu_n(\vx) = [\mu_{n,j}(\vx)]^\top_{j \leq d_s}$ and
 $\vsigma_n(\vx) = [\sigma_{n,j}(\vx)]^\top_{j \leq d_s}$.

\begin{corollary}\label{cor: upper bound on cost}
Let \cref{ass:rkhs-func} hold, then we have for all $\vpi \in \Pi$, $n \geq 0$, with probability at least $1-\delta$
\begin{equation*}
    P_n(\vpi) \geq J_c(\vpi, \vf^*) 
\end{equation*}
\end{corollary}
\begin{proof}
    Note that $\vf^* \in \gQ_n$ for all $n \geq 0$ with probability at least $1-\delta$ (\cref{lem:rkhs-confidence-interval}). Therefore, $\vf^* \in \gM_n$. Furthermore, 
    \begin{align*}
        P_n(\vpi) &= \max_{\vf \in \gM_n} J_c(\vpi, \vf) \\
        &\geq J_c(\vpi, \vf^*)
    \end{align*}
\end{proof}

\begin{lemma}[Difference in Policy performance, \citet{sukhija2024optimistic}]
Consider any function $r: \gS \times \gA \to \R$.
Let $J_{r, k}(\vpi, \vf^*, \vs_k) =  \E_{\vtau^{\vpi}}\left[\sum_{t=k}^{T-1} r(\vs_t, \vpi(\vs_{t}))\right]$ and $A_{r, k}(\vpi, \vs, \va) =  \E_{\vtau^{\vpi}}\left[r(\vs, \va) + J_{r, k + 1}(\vpi, \vf^*, \vs') - J_{r, k}(\vpi, \vf^*, \vs) \right]$ with $\vs' = \vf^*(\vs, \va) + \vw$. For simplicity we refer to  $J_{r, 0}(\vpi, \vf^*,\vs_0) = J_r (\vpi, \vf^*, \vs_0)$. 
The following holds for all $\vs_0 \in \gS$:
\begin{equation*}
    J_r (\vpi', \vf^*, \vs_0) -  J_r (\vpi, \vf^*,\vs_0) = \E_{\vtau^{\vpi'}} \left[\sum^{T-1}_{t=0} A_{r, t}(\vpi, \vs'_t, \vpi'(\vs'_t))\right]
\end{equation*}
\label{lemma:difference-in-policy}
\end{lemma}
\begin{proof}
See Lemma 5. \citet{sukhija2024optimistic}.
\end{proof}

\begin{lemma}[Comparing safety costs of policies]
    \begin{equation*}
    J_c (\vpi, \vf^*, \vs_0) -  J_c (\vpi', \vf^*, \vs_0) \leq D(\vpi, \vpi')
\end{equation*}
\label{lem:safety-cost-comparison}
\end{lemma}
\begin{proof}
For notational convenience we will omit the dependance on $\vf^*$. 
    \begin{align*}
         &J_c (\vpi, \vs_0) - J_c (\vpi', \vs_0) = \E_{\vtau^{\vpi'}} \left[\sum^{T-1}_{t=0} -A_{c, t}(\vpi, \vs'_t, \vpi'(\vs'_t))\right] \\
         &= \E_{\vtau^{\vpi'}} \left[\sum^{T-1}_{t=0} - (c(\vs_t, \vpi'(\vs_t)) - c(\vs_t, \vpi(\vs_t)))\right] \\ 
         &+ 
         \E_{\vtau^{\vpi'}} \left[\sum^{T-1}_{t=0}\E_{\vs'_{t+1}|\vs_t, \vpi(\vs_t)}\left[
         J_{c, k + 1}(\vpi, \vs'_{t+1})\right] - \E_{\vs'_{t+1}|\vs_t, \vpi'(\vs_t)}\left[
         J_{c, k + 1}(\vpi, \vs'_{t+1})\right]
         \right] \\
         &= \E_{\vtau^{\vpi'}}\left[\sum^{T-1}_{t=0} 
         c(\vs_t, \vpi(\vs_t)) - c(\vs_t, \vpi'(\vs_t))\right] \\
         &+ 
         \E_{\vtau^{\vpi'}}\left[\sum^{T-1}_{t=0}\E_{\vs'_{t+1}|\vs_t, \vpi(\vs_t)}\left[
         J_{c, k + 1}(\vpi, \vs'_{t+1})\right] - \E_{\vs'_{t+1}|\vs_t, \vpi'(\vs_t)}\left[
         J_{c, k + 1}(\vpi, \vs'_{t+1})\right]\right] \\
        &\leq \E_{\vtau^{\vpi'}}\left[\sum^{T-1}_{t=0} 
         \min\left\{L_c\norm{\vpi'(\vs_t) - \vpi(\vs_t)}, 2C_{\max}\right\}\right] \\ 
         &+  \E_{\vtau^{\vpi'}}\left[\sum^{T-1}_{t=0} \sqrt{\E_{\vs'_{t+1}|\vs_t, \vpi(\vs_t)}\left[
         J^2_{c, k + 1}(\vpi, \vs'_{t+1})\right]} \min\left\{\frac{\norm{\vf^*(\vs_t, \vpi(\vs_t)) - \vf^*(\vs_t, \vpi'(\vs_t))}}{\sigma}, 1\right\}\right] \tag*{\citep[Lemma C.2.]{kakade2020information}}\\
         &\leq  \E_{\vtau^{\vpi'}}\left[\sum^{T-1}_{t=0} 
        \min\left\{L_c\norm{\vpi'(\vs_t) - \vpi(\vs_t)}, 2C_{\max}\right\} + TC_{\max} \min\left\{\frac{L_\vf\norm{\vpi'(\vs_t) - \vpi(\vs_t)}}{\sigma}, 1\right\}\right] \\
        &= D(\vpi, \vpi')
    \end{align*}
\end{proof}
\begin{lemma}
    Let \cref{ass:lipschitz-continuity} -- \cref{ass:rkhs-func} hold. Then we have $\forall n \geq 0$,  $\vpi \in \gS_n \setminus \gS_{n-1}$ with probability at least $1-\delta$,
        $J_c(\vpi) \leq d$.
    \label{lem:safety-of-newly-added-policies}
\end{lemma}
\begin{proof}
    Consider any $\vpi \in \gS_n \setminus \gS_{n-1}$. By \cref{def:safeset}, we have that there exists a $\vpi'$ in $\gS_{n-1}$ such that
    \begin{equation*}
        P_n(\vpi') + D(\vpi, \vpi') \leq d
    \end{equation*}
    Therefore,
    \begin{align*}
       d &\geq  P_n(\vpi') + D(\vpi, \vpi') \\
       &\geq J_c(\vpi', \vf^*) + D(\vpi, \vpi') \tag{\cref{cor: upper bound on cost}}\\
       &\geq J_c(\vpi, \vf^*) \tag{\cref{lem:safety-cost-comparison}}.
    \end{align*}
\end{proof}

\begin{corollary}[All policies in $\gS_n$ are safe]
 Let \cref{ass:lipschitz-continuity} -- \cref{ass:rkhs-func} hold. Then we have $\forall n \geq 0$,  $\vpi \in \gS_n$ with probability at least $1-\delta$, $J_c(\vpi) \leq d$.
 \label{corollary:all-policies-are-safe}
\end{corollary}
\begin{proof}
    We prove this by induction. For $n=0$, this holds due to \cref{ass:safe-seed}.
    Consider any $n > 0$. By induction, $\forall \vpi \in  \gS_n$ we have that $J_c(\vpi, \vf^*) \leq d$. Hence, we focus on $\vpi \in \gS_{n+1} \setminus \gS_n$. In \cref{lem:safety-of-newly-added-policies}, we show $J_c(\vpi, \vf^*) \leq d$ for all $\vpi \in \gS_{n+1} \setminus \gS_n$. This completes the proof.
\end{proof}

\begin{lemma}
    Consider any positive and bounded function $c \in [0, C_{\max}]$. Let \cref{ass:lipschitz-continuity} -- \ref{ass:rkhs-func} hold. Then we have $\forall n \geq 0$, $\forall \vf \in \gM_n$. with probability at least $1-\delta$
    \begin{equation*}
        |J_c(\vpi, \vf) - J_c(\vpi, \vf^*)| \leq TC_{\max}  \E_{\vtau^{\vpi}} \left[\sum^{T-1}_{t=0} \frac{(1 + \sqrt{d_s}) \beta_{n-1}(\delta) \norm{\vsigma_{n-1}(\vs_t, \vpi(\vs_t))}}{\sigma}\right].
    \end{equation*}
    \label{lem:bound-on-cost}
\end{lemma}
\begin{proof}
From \citet[Corollary 2.]{sukhija2024optimistic} we have,
    \begin{align*}
        J_c(\vpi, \vf) - J_c(\vpi, \vf^*) &= \E_{\vtau^{\vpi}}\left[\sum^{T-1}_{t=0} J_{c, t+1}(\vpi, \vf, \vs_{t+1}) -  J_{c, t+1}(\vpi, \vf, \vs'_{t+1})\right], \tag{Expectation w.r.t $\vpi$ under true dynamics $\vf^*$} \\
    &\text{with } \vs_{t+1} = \vf^*(\vs_{t}, \vpi(\vs_t)) + \vw_t, \\
    &\text{and } \vs'_{t+1} = \vf(\vs_{t}, \vpi(\vs_t))+ \vw_t.
    \end{align*}
    Furthermore, $J_{c, t+1}(\vpi, \vf, \vs) \in [0, TC_{\max}]$ for all $\vpi, \vf, \vs$, and $t$. Therefore, given $\vs_t$,
    \begin{align*}
        &\left|\E_{\vw_t}\left[J_{c, t+1}(\vpi, \vf, \vs'_{t+1}) -  J_{c, t+1}(\vpi, \vf, \vs_{t+1})\right]\right| \\
        &\leq \max\left\{\sqrt{\E_{\vw_t}[J^2_{c, t+1}(\vpi, \vf, \vs'_{t+1})]},  \sqrt{\E_{\vw_t}[J^2_{c, t+1}(\vpi, \vf, \vs_{t+1})]}\right\}\min\left\{\frac{\norm{\vf^*(\vs_t, \vpi(\vs_t)) - \vf(\vs_t, \vpi(\vs_t))}}{\sigma}, 1\right\} \tag*{ \citep[Lemma C.2.]{kakade2020information}}\\
        &\leq T C_{\max}\min\left\{\frac{\norm{\vf^*(\vs_t, \vpi(\vs_t)) - \vf(\vs_t, \vpi(\vs_t))}}{\sigma}, 1\right\} \\
        &\leq T C_{\max} \min\left\{\frac{(1 + \sqrt{d_s}) \beta_{n-1}(\delta) \norm{\vsigma_{n-1}(\vs_t, \vpi(\vs_t))}}{\sigma}, 1\right\} \tag*{\citep[Corollary 3]{sukhija2024optimistic}}
    \end{align*}
\end{proof}

From hereon let $C = \frac{(1 + \sqrt{d_s}) \max\{R_{\max}, C_{\max}, \sigma_0\}}{\sigma}$.

\begin{lemma}
     Let \cref{ass:lipschitz-continuity} -- \ref{ass:rkhs-func} hold. Then we have $\forall n, N \geq 0$ with probability at least $1-\delta$
     \begin{equation*}
          \max_{\vpi \in \gS_n} \E_{\vtau^{\vpi}} \left[\sum^{T-1}_{t=0} \norm{\vsigma_{N + n-1}(\vs_t, \vpi(\vs_t))}\right]
        \leq T^2 C\frac{\sqrt{d_s} \sigma_0}{\sqrt{\log(1 + \sigma^{-2}\sigma^2_0})}\sqrt{\frac{\beta^2_{n+N-1}(\delta)\gamma_{n+N-1}(k)}{N}}.
     \end{equation*}
     \label{lemm: bound on sigmas}
\end{lemma}
\begin{proof}
Consider any $N > 0$,
    \begin{align*}
        \max_{\vpi \in \gS_n} \E_{\vtau^{\vpi}} \left[\sum^{T-1}_{t=0} \norm{\vsigma_{N + n-1}(\vs_t, \vpi(\vs_t))}\right]
        &\leq 
        \frac{1}{N}\sum^{N-1}_{i=0}  \max_{\vpi \in \gS_n} \E_{\vtau^{\vpi}} \left[\sum^{T-1}_{t=0} \norm{\vsigma_{n + i}(\vs_t, \vpi(\vs_t))}\right] \tag{Monotonocity of the variance} \\
        &\leq \frac{1}{N}\sum^{N-1}_{i=0}  \max_{\vpi \in \gS_{n+i}} \E_{\vtau^{\vpi}} \left[\sum^{T-1}_{t=0} \norm{\vsigma_{n + i}(\vs_t, \vpi(\vs_t))}\right] \tag{Monotonocity of the safe set} \\
        &=\frac{1}{N}\sum^{n+N-1}_{i=n} \E_{\vtau^{\vpi^*_{i}}} \left[\sum^{T-1}_{t=0} \norm{\vsigma_{i}(\vs_t, \vpi^*_{i}(\vs_t))}\right] \tag{Definition of $\vpi^*_i$} \\
        &=\frac{1}{N}\sum^{n+N-1}_{i=n} \E_{\vtau^{\vpi_{i}}} \left[\sum^{T-1}_{t=0} \norm{\vsigma_{i}(\vs_t, \vpi_{i}(\vs_t))}\right] + \frac{1}{N}(J(\vpi^*_i) - J(\vpi_i))
    \end{align*}
    Let $r_i = J_i(\vpi^*_i) - J(\vpi_i)$, where $\vpi_i$ is the policy proposed by \cref{eq:exploration-op-optimistic}. We analyze this regret term. 
    Note that since, we optimistically pick dynamics from $\gM_n$, we have $J_i(\vpi^*_i) \leq J(\vpi_i, \vf_i)$, where $\vf_i$ are the optimistic dynamics. Therefore, $r_i \leq J(\vpi_i, \vf_i) - J(\vpi_i, \vf^*)$. Hence, we can invoke \cref{lem:bound-on-cost} to get
    \begin{equation*}
        r_i \leq T C \left[\sum^{T-1}_{t=0} \beta_{i}(\delta) \norm{\vsigma_{i}(\vs_t, \vpi(\vs_t))}\right].
    \end{equation*}
    Therefore,
    \begin{align*}
         \max_{\vpi \in \gS_n} \E_{\vtau^{\vpi}} \left[\sum^{T-1}_{t=0} \norm{\vsigma_{N + n-1}(\vs_t, \vpi(\vs_t))}\right]
        &\leq \frac{ T C \beta_{n+N-1}(\delta)}{N}\sum^{n+N-1}_{i=n} \E_{\vtau^{\vpi_{i}}} \left[\sum^{T-1}_{t=0}  \norm{\vsigma_{i}(\vs_t, \vpi_{i}(\vs_t))}\right] \\
        &\leq \frac{T C \beta_{n+N-1}(\delta)}{N}\sum^{n+N-1}_{i=n} \E_{\vtau^{\vpi_{i}}} \left[\sum^{T-1}_{t=0}  \norm{\vsigma_{i}(\vs_t, \vpi_{i}(\vs_t))}\right] \\
        &\leq  \frac{T C \beta_{n+N-1}(\delta)}{N} \sqrt{N T} \sqrt{\sum^{n+N-1}_{i=n} \E_{\vtau^{\vpi_{i}}} \left[\sum^{T-1}_{t=0}  \norm{\vsigma_{i}(\vs_t, \vpi_{i}(\vs_t))}^2\right]} \tag{Cauchy-Schwartz}\\
        &\leq \frac{T C \beta_{n+N-1}(\delta)}{N} \sqrt{N T} \sqrt{\sum^{n+N-1}_{i=0} \E_{\vtau^{\vpi_{i}}} \left[\sum^{T-1}_{t=0}  \norm{\vsigma_{i}(\vs_t, \vpi_{i}(\vs_t))}^2\right]} \\
        &\leq \frac{T C \frac{\sqrt{Td_s} \sigma_0}{\sqrt{\log(1 + \sigma^{-2}\sigma^2_0})} \beta_{n+N-1}(\delta)}{N} \sqrt{N T} \sqrt{\gamma_{n+N-1}(k)} \tag*{\citep[Lemma 17]{curi2020efficient}}\\
        &=  T^2 C\frac{\sqrt{d_s} \sigma_0}{\sqrt{\log(1 + \sigma^{-2}\sigma^2_0})}\sqrt{\frac{\beta^2_{n+N-1}(\delta)\gamma_{n+N-1}(k)}{N}}
    \end{align*}
    \end{proof}
    
    \begin{lemma}
    Let \cref{ass:lipschitz-continuity} -- \ref{ass:rkhs-func} hold and define $N_n$ to be the smallest integer such that 
    \begin{equation*}
        T^{3} C^2 \frac{\sqrt{d_s} \sigma_0}{\sqrt{\log(1 + \sigma^{-2}\sigma^2_0})}  \beta^2_{n + N_n-1}(\delta) \sqrt{\frac{\gamma_{n+N_n-1}(k)}{N_n}} \leq \epsilon.
    \end{equation*}
   Then, we have $\forall \vpi \in \gS_n$, $\vf \in \gM_{n+N_n-1}$ with probability at least $1-\delta$
   \begin{equation*}
        |J_c(\vpi, \vf) - J_c(\vpi, \vf^*)| \leq \epsilon, \; \text{and,} \; |J_r(\vpi, \vf) - J_r(\vpi, \vf^*)| \leq \epsilon.
   \end{equation*}

   \label{lemma: bound on uncertainty in S_n}
    \end{lemma}
 \begin{proof}
    \begin{align*}
         |J_c(\vpi, \vf) - J_c(\vpi, \vf^*)| &\leq TC_{\max}  \E_{\vtau^{\vpi}} \left[\sum^{T-1}_{t=0} \frac{(1 + \sqrt{d_s}) \beta_{n + N_n-1}(\delta) \norm{\vsigma_{n + N_n-1}(\vs_t, \vpi(\vs_t))}}{\sigma}\right] \tag{\cref{lem:bound-on-cost}}\\
         &\leq TC \beta_{n + N_n-1}  \left[\sum^{T-1}_{t=0} \norm{\vsigma_{n + N_n-1}(\vs_t, \vpi(\vs_t))}\right] \\
         &\leq TC \beta_{n + N_n-1}  T^2 C\frac{\sqrt{d_s} \sigma_0}{\sqrt{\log(1 + \sigma^{-2}\sigma^2_0})} \beta_{n+N_n-1}(\delta)  \sqrt{\frac{\gamma_{n+N_n-1}(k)}{N_n}} \tag{\cref{lemm: bound on sigmas}}\\
         &\leq \epsilon
    \end{align*}
    We can apply the same inequalities for $J_r$.
\end{proof}

\begin{corollary}
    Let \cref{ass:lipschitz-continuity} -- \ref{ass:rkhs-func} hold. Consider any $n\geq0$ and define $N_n$ as in \cref{lemma: bound on uncertainty in S_n}. Then we have  with probability at least $1-\delta$
    \begin{equation*}
        \gS_{n+N_n} \supseteq \gR^{\varepsilon}(\gS_n).
    \end{equation*}
    \label{cor:expansion-of-safe-set}
\end{corollary}
\begin{proof}
    From \cref{lemma: bound on uncertainty in S_n}, we have $\forall \vpi \in \gS_n$, $\vf \in \gM_{n+N_n-1}$,
        $|J_c(\vpi, \vf) - J_c(\vpi, \vf^*)| \leq \epsilon$, therefore $P_{n+N_n - 1}(\vpi) \leq J_c(\vpi, \vf^*) + \epsilon$.
        For the sake of contradiction, assume there exists a policy $\vpi \in \gR^{\varepsilon}(\gS_n) \setminus \gS_{n+N_n}$. We study the case where $\vpi \in  \gR^{\varepsilon}(\gS_n) \setminus \gS_n$ else we have a contradiction ($\gS_n \subseteq \gS_{n+N_n}$).
        Since $\vpi \in  \gR^{\varepsilon}(\gS_n) \setminus \gS_n$, there exists a $\vpi' \in \gS_n$ such that
        $J_c(\vpi') + D(\vpi, \vpi') \leq d -\epsilon$ (see~\cref{thm:main-theorem}). Hence, we get 
        \begin{align*}
            d &\geq J_c(\vpi') + \epsilon + D(\vpi, \vpi') \\
            &\geq P_{n+N_n -1 }(\vpi') + D(\vpi, \vpi').
        \end{align*}
        Since, $\vpi' \in \gS_n \subseteq \gS_{n + N_n -1}$, by the definition of the safe set (c.f. ~\cref{def:safeset}), this implies that $\vpi \in \gS_{n+N_n}$, which is a contradiction.
\end{proof}

A key property of $N_n$ is that it increases monotonously with $n$. Moreover, for a given $n\geq 0$, $N_n$ is the smallest integer satisfying
\begin{equation*}
     N_n \geq \frac{\gamma_{n+N_n-1}(k)                \beta^4_{n + N_n-1}(\delta) T^{6} C^4\frac{d_s \sigma^2_0}{\log(1 + \sigma^{-2}\sigma^2_0)}}{\epsilon^2}.
\end{equation*}
Both functions $n \mapsto \gamma_n$, and $n \mapsto \beta_n$ are monotonically increasing with $n$. Hence increasing $n$, increases the right-hand side of the inequality, and therefore $N_n$.
\begin{lemma}
      Let \cref{ass:lipschitz-continuity} -- \ref{ass:rkhs-func} hold and consider $n^{*} \geq (H + 1) N_{n^*}$. Then we have with probability at least $1-\delta$ for all $n\geq n^*$ 
    \begin{equation*}
        S_{n} \supseteq \gR^{\varepsilon}_H(\gS_0).
    \end{equation*}
    \label{lemma:convergence-to-the-reachable-safe-set}
\end{lemma}
\begin{proof}
To prove this, we show for any positive integer $ k \leq H$, that 
$S_{kN_{n^*}} \supseteq \gR^{\varepsilon}_k(\gS_0)$ by induction. 

Moreover for any $k$, let $T_{k} = T_{k-1} + N_{T_{k-1}}$ and $T_0 = 0$. We inductively show that $T_k \leq kN_{n^*}$ for all $k \leq H$. 

For the base case $k=1$, we have $T_1 = N_0 \leq N_{n^*}$ since $n^* \geq 0$.
Consider any $k \leq H$, then, we have $T_k = T_{k-1} + N_{T_{k-1}}$. By induction $T_{k-1} \leq (k-1) N_{n^*}$. Therefore, $T_k \leq (k-1) N_{n^*} + N_{(k-1) N_{n^*}}$. Furthermore, note that $(k-1) N_{n^*} \leq n^*$ for all $k \leq H$. Therefore, $T_k \leq (k-1) N_{n^*} + N_{n^*} = k N_{n^*}$.

Next, we have from \cref{cor:expansion-of-safe-set}, $\gS_{T_k} \supseteq \gR^{\varepsilon}(\gS_{T_{k-1}}) \coloneqq \gR^{\varepsilon}_k(\gS_0)$.
Moreover, $\gS_{T_1} \coloneqq \gS_{N_0} \supseteq \gR^{\varepsilon}(\gS_0)$. Similarly, $\gS_{T_2} \coloneqq \gS_{N_0 + N_{N_0}} \supseteq \gR^{\varepsilon}(\gS_1) \coloneqq \gR^{\varepsilon}_2(\gS_0)$, etc. Therefore, we get $\gS_{HN_{n^*}} \supseteq \gS_{T_{H}} \supseteq \gR^{\varepsilon}_{H}(\gS_0)$. As $n^* \geq HN_{n^*}$, this completes the proof.
\end{proof}

\begin{lemma}
    Let \cref{ass:lipschitz-continuity} -- \ref{ass:rkhs-func} hold and consider the smallest integer $n^{*}$ such that
\begin{equation}
            \frac{n^*}{\gamma_{n^*}(k)                \beta^4_{n^*}(\delta)} \geq \frac{(H + 1) T^{6} C^4\frac{d_s \sigma^2_0}{\log(1 + \sigma^{-2}\sigma^2_0)}}{\epsilon^2}.
        \end{equation}
        Then we have for all $n\geq n^*$ 
    \begin{equation*}
        S_{n} \supseteq \gR^{\varepsilon}_H(\gS_0).
    \end{equation*} 
    Moreover, we have 
    for all $n \geq n^*$, $\vpi \in \gR^{\varepsilon}_H(\gS_0)$ that 
    $|J_c(\vpi, \vf) - J_c(\vpi, \vf^*)| \leq \epsilon$ and  $|J_r(\vpi, \vf) - J_r(\vpi, \vf^*)| \leq \epsilon$.
    \label{lemma: convergence to the reachable safe set}
    \end{lemma}
    \begin{proof}
    Note that for any $n$, $N_n$ is defined as the smallest integer satisfying:
        \begin{align*}
 \frac{N_n}{\gamma_{n+N_n-1}(k)                \beta^4_{n + N_n-1}(\delta)} \geq \frac{T^{6} C^4\frac{d_s \sigma^2_0}{\log(1 + \sigma^{-2}\sigma^2_0)}}{\epsilon^2}.
        \end{align*}
        From \cref{lemma:convergence-to-the-reachable-safe-set}, for $n^* = (H + 1) N_{n^*}$, we have for all $n\geq n^*$ 
    \begin{equation*}
        S_{n} \supseteq \gR^{\varepsilon}_H(\gS_0).
    \end{equation*}
        We show that the solution to \cref{eq: bound on sample complexity} satisfies this condition. Moreover, let $n^*  = (H+1)N_{n^*}$
        \begin{align*}
             \frac{N_{n^*}}{\gamma_{n^*+N_{n^*}-1}(k)                \beta^4_{n^* + N_{n^*}-1}(\delta)} &=  \frac{\frac{n^*}{H+1}}{\gamma_{n^*+\frac{n^*}{H+1}-1}(k)                \beta^4_{n^* + \frac{n^*}{H+1}-1}(\delta)} \\
             &\geq \frac{\frac{n^*}{H+1}}{\gamma_{n^*}(k)                \beta^4_{n^*}(\delta)}
        \end{align*}
        Picking $n^*$ as the smallest integer satisfying 
        \begin{equation*}
            \frac{n^*}{\gamma_{n^*}(k)                \beta^4_{n^*}(\delta)} \geq \frac{(H+1) T^{6} C^4\frac{d_s \sigma^2_0}{\log(1 + \sigma^{-2}\sigma^2_0)}}{\epsilon^2},
        \end{equation*}
        ensures that 
        \begin{align*}
 \frac{N_{n^*}}{\gamma_{{n^*}+N_{n^*}-1}(k)                \beta^4_{{n^*}+ N_{n^*}-1}(\delta)} \geq \frac{T^{6} C^4\frac{d_s \sigma^2_0}{\log(1 + \sigma^{-2}\sigma^2_0)}}{\epsilon^2}
        \end{align*}

Finally, from \cref{lemma: bound on uncertainty in S_n} we have that $S_{H N_{n^*}} \supseteq \gR^{\varepsilon}_H(\gS_0)$.

Therefore, for all $n \geq n^*$, 
$\vpi \in \gR^{\varepsilon}_H(\gS_0)$, $\vf \in \gM_{n^*+N_{n^*}-1}$ with probability at least $1-\delta$
   \begin{equation*}
        |J_c(\vpi, \vf) - J_c(\vpi, \vf^*)| \leq \epsilon, \; \text{and,} \; |J_r(\vpi, \vf) - J_r(\vpi, \vf^*)| \leq \epsilon.
   \end{equation*}
    \end{proof}

    \begin{proof}[Proof of \cref{thm:main-theorem}]
We prove in \cref{corollary:all-policies-are-safe}, that all policies in $\gS_n$ are safe for all $n\geq 0$.
\algo{} is safe since it picks policies only from $\gS_n$.

For optimality, we showed in \cref{lemma:convergence-to-the-reachable-safe-set} for all $n \geq n^*$ that $ S_{n} \supseteq \gR^{\varepsilon}_H(\gS_0)$.
Moreover, we have $\forall \vpi \in \gR^{\varepsilon}_H(\gS_0)$, $\vf \in  \gM_{n^*+N_{n^*}-1} \supseteq \gM_{n}$, 
$|J_r(\vpi, \vf) - J_r(\vpi, \vf^*)| \leq \epsilon$.
Let $\vpi^*$ be the optimal policy and let $\tilde{\vpi}_n$ denote the solution to $\arg\max_{\vpi \in \gS_n} \min_{\vf \in \gM_n} J_r(\vpi, \vf)$. 
For the sake of contradiction, assume that 
\begin{equation}
\label{eq:contradictory_ass}
 J_r(\tilde{\vpi}_n)  < \max_{\vpi \in \gR^{\epsilon}_H(\gS_0)} J_r(\vpi, \vf^*) - \epsilon.
\end{equation}
Furthermore, let $P^r_n(\vpi) = \min_{\vf \in \gM_n} J_r(\vpi, \vf)$ for all $\vpi \in \Pi$.
\begin{align*}
 P^r_n(\vpi^*) &\leq \max_{\vpi \in \gS_n} P^r_n(\vpi) \\
 &= P^r_n(\tilde{\vpi}_n) \\
&\leq J_r(\tilde{\vpi}_n) \\
    &<  J_r(\vpi^*, \vf^*) - \epsilon \tag{contradiction assumption}\\
    &\leq  P^r_n(\vpi^*) \tag{\cref{lemma: convergence to the reachable safe set}}.
\end{align*}
This is a contradiction, which completes the proof.

\end{proof}
\section{Experiment Details}
\label{sec:experiment-details}
\subsection{GP Experiments}
For the GP experiments, we approximate \cref{eq:exploration-op-optimistic-practical} with the following unconstrained optimization problem.
\begin{align}
 \argmax_{\vpi \in \Pi} \max_{\vf \in \gQ_n} J_n(\vpi, \vf) - \lambda \max\left\{\max_{\vf' \in \gQ_n} J_c(\vpi, \vf') - d, 0\right\}.
\end{align}
Here $\lambda$ is a (large) penalty term that is used to discourage constraint violation. 
We use the iCEM~\citep{iCem} optimizer to solve the constrained optimization above.
Effectively, given a sequence of actions $\{\va_t\}^{H}_{t=0}$, we roll them out on our learned GP model using the TS1 approach from \citet{chua2018deep}. 
Moreover, we maintain $P$ particles, and given the state $(\vs^{p}_t, \va^{p}_t)$ for the $p$-th particle, we determine the next state $\vs^{p}_{t+1}$, by sampling from $\gN(\vmu_n(\vs^{p}_t, \va^{p}_t), \vsigma_n(\vs^{p}_t, \va^{p}_t))$. Accordingly, for each action sequence $\{\va_t\}^{H}_{t=0}$, we obtain $P$ trajectories and we empirically solve $\max_{\vf' \in \gQ_n} J_c(\vpi, \vf')$ by taking the max over the $P$ trajectories. This approach is also proposed by \citet{kakade2020information} as a heuristic for optimizing over the dynamics. 

\paragraph{Rewards and constraints}
The reward function is designed to penalize deviations in both the angular position and the control input from the target behavior. For both the \textsc{Pendulum} and \textsc{Cartpole}, the state of the pole can be defined as follows. Let $\theta$ be the current angle, $\omega$ the angular velocity, and $u$ the control input. The target angle is denoted as $\theta_{\text{target}}$, and the angular error between the current angle and the target angle is $\Delta \theta$. The reward and cost functions for the \textsc{Pendulum} environment are given by
\begin{equation*}
    r_{\text{Pendulum}} = -\left(\Delta \theta^2 + 0.1 \cdot \omega^2 + 0.02 \cdot u^2 \right), \quad c_{\text{Pendulum}} = \max\{\abs{\omega}- 6.0, 0.0\}, d = 0.0.
\end{equation*}
For the \textsc{Cartpole} environment, the position and velocity of the slider are defined as $p$ and $v$ respectively.
The reward for the \textsc{Cartpole} environment is the given by
\begin{equation*}
    r_{\text{Cartpole}} = -\left(\Delta \theta^2 + p^2 + 0.1 \cdot \left( v^2 + \omega^2 \right) \right) - 0.01 \cdot u^2, \quad c_{\text{Cartpole}} = \max\{\abs{p}- 0.5, 0.0\}, d = 0.75
    .
\end{equation*}
\subsection{Vision Control Experiments}
We provide an open-source implementation of our experiments in \url{https://github.com/yardenas/actsafe}. We encourage readers to use it, as it contains additional important implementation details. In all the experiments below, our policy consists of 750K parameters, a several orders of magnitudes compared to previous works on provable safe explorations.

\paragraph{Approximating \Cref{eq:exploration-op-optimistic-practical}}
We solve the constraint optimization problem in~\Cref{eq:exploration-op-optimistic-practical} using the LBSGD solver from~\citet{usmanova2024log}.
LBSGD is a first-order optimizer that
uses a logarithmic barrier function to enforce constraint satisfaction. Previous works from~\citet{ni2024safeexplorationapproachconstrained, as2024safeexplorationusingbayesian} have successfully applied LBSGD for planning in model-based RL with CMDPs, showing notably fewer constraint violations than alternative solvers like the augmented Lagrangian method~\citep{as2022constrained}.

To approximate \Cref{eq:exploration-op-optimistic-practical}, we maintain an RSSM ensemble of $P$ particles and given the state action pair $(\vs_t, \pi_n(\va_t | \vs_t))$, we obtain $P$ estimates $\{\vs^p_{t + 1}\}_{p = 1}^P$ for the next state. We estimate $\vsigma^2_n$ with the variance/disagreement between the ensemble members, i.e.,  $\text{Var}\left(\{\vs^p_{t + 1}\}_{p = 1}^P\right)$. We obtain the next state $\vs_{t+1}$ by uniform sampling from $\{\vs^p_{t + 1}\}_{p = 1}^P$, i.e., TS1 from \citet{chua2018deep}. 
Akin to \citet{yu2020mopo}, we approximate $\max_{\vf' \in \gQ_n} J_c(\vpi, \vf')$ by penalizing the cost function with $\vsigma_n$
\begin{equation*}
    J_{c -\lambda \vsigma}(\vpi_n) = \E_{\vpi_n}\left[\sum^{H}_{t=0}\gamma^t(c(\vs_t, \va_t) + \lambda \norm{\vsigma_{n}(\vs_t, \va_t)})\right],
\end{equation*}
where $\lambda$ is a pessimism parameter.  \citet{yu2020mopo} show that for an appropriate choice of $\lambda$, $J_{c -\lambda \vsigma}(\vpi_n)$ is indeed a pessimistic estimate of $J_{c}(\vpi_n)$. However, in our experiments we treat $\lambda$ as a hyper-parameter.

\paragraph{Safety experiments}
We focus on \sgym{} to showcase our practical algorithm design maintains constraint satisfaction during learning. Our experiments rely on a newer fork of \sgym{} which is available via our open-source code. We follow the experimental setup of \citet{Ray2019,as2022constrained} and an episode length of $T = 1000$. We set the cost budget for each episode to $d = 25$ for \sgym{} \citep[see][]{Ray2019}. After each training epoch we estimate $J_r(\vpi_n)$ and $J_c(\vpi_n)$ by sampling 50 episodes, denoting the estimates with $\hat{J}_r$ and $\hat{J}_c$. Unless specified otherwise, in all our experiments we use 5 random seeds and report the median and standard error across these seeds. Finally, we use a budget of 5M training steps for each training run. To make a fair comparison with \citet{as2022constrained,huang2024safedreamer}, we fix the ratio of environment steps and update steps of the model and policy. While \citet{huang2024safedreamer} use the RSSM model from \citet{hafner2023mastering}, our implementation uses the (older) one from \citet{hafner2019planet} and \citet{as2022constrained}.

\paragraph{Sparse \sgym{}}
\newcommand{\distance}[2]{d_{#1}^{\text{#2}}}
Let $\distance{t}{RG}$ be the euclidean distance between the robot and the goal/button at time step $t$, $\distance{t}{BG}$ the distance between the box and the goal position and $\distance{t}{RB}$ the distance between the robot and the box positions. Furthermore, denote $\text{tol}(x, l, u)$ as the tolerance function from \citet{tassa2018deepmind}, where $l, u$ denotes lower and upper bounds respectively.
\begin{table}[h!]
\centering
\begin{tabularx}{\textwidth}{lXX} 
\toprule
\textbf{Environment} & \textbf{Dense Reward} & \textbf{Sparse Reward} \\
\midrule
\textsc{GotoGoal} & $\distance{t - 1}{RG} - \distance{t}{RG} + \1_{\distance{t}{BG} \leq 0.3}$ & $\text{tol}(\distance{t}{RG}, 0, 0.45)\cdot(\distance{t - 1}{RG} - \distance{t}{RG}) + \1_{\distance{t}{RG} \leq 0.3}$ \\ \midrule
\textsc{PressButton} & $\distance{t - 1}{RG} - \distance{t}{RG} + \1_{\distance{t}{BG} \leq 0.3}$ & $\1_{\distance{t}{RG} \leq 0.1}$  \\ \midrule
\textsc{PushBox} & $\distance{t - 1}{RB} - \distance{t}{RB} + \distance{t - 1}{BG} - \distance{t}{BG} + \1_{\distance{t}{BG} \leq 0.3}$&  $\text{tol}(\distance{t}{RB}, 0, 0.5)\cdot(\distance{t - 1}{RB} - \distance{t}{RB}) + \distance{t - 1}{BG} - \distance{t}{BG} + \1_{\distance{t}{BG} \leq 0.3}$ \\
\bottomrule
\end{tabularx}
\caption{Comparison of the reward functions in the base environments of \sgym{} and our sparse rewards environments.}
\vspace{-0.5cm}
\label{tab:dense-sparse-rewards}
\end{table}

\paragraph{Cartpole exploration}
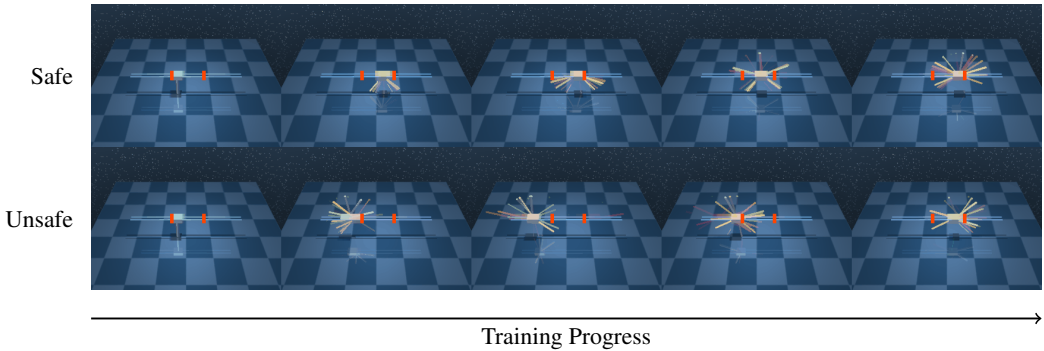
\begin{figure}[ht]
\centering
\resizebox{\textwidth}{!}{%
    \begin{tikzpicture}
      \node[anchor=south west,inner sep=0] (image) at (0,0) {\includegraphics[width=\textwidth]{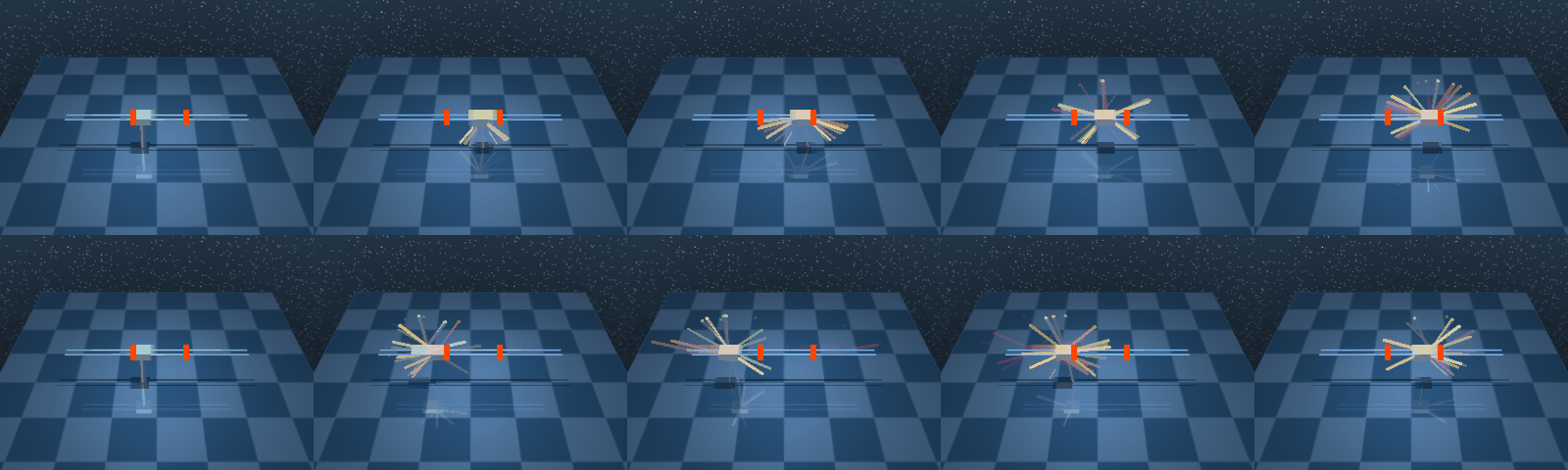}};
      \begin{scope}[x={(image.south east)},y={(image.north west)}]
        \node[anchor=east] at (-0.01,0.75) {Safe};
        \node[anchor=east] at (-0.01,0.25) {Unsafe};
        \draw[->,thick] (0,-0.1) -- (1,-0.1) node[midway,below] {Training Progress};
      \end{scope}
    \end{tikzpicture}
}
\caption{\textsc{Cartpole} environment as an example of a problem instance of safe exploration. Each scene summarizes a trajectory that was collected in increasing training iterations. The agent incurs a cost whenever the cart goes outside of the area between the two red vertical lines. The goal is to learn a policy that swings the pole to the top position, while ensuring the expected accumulated cost is bounded \emph{during learning}. Learning in this setting is much more challenging, as agents can only try out control policies that known to be safe.}
\label{fig:cartpole-demo}
\end{figure}
In this task, the agent receives a sparse reward when it swings up a pendulum to the top position and when the slider (a.k.a cart) is centered. The \rwrl{} benchmark \citep{dulacarnold2019challengesrealworldreinforcementlearning} adds a safety constraint that enforces the slider to remain in a certain distance from the center (see \Cref{fig:cartpole-demo}). As in \citet{dulacarnold2019challengesrealworldreinforcementlearning}, we use a cost budget of $d = 100$ and an episode length of $T = 1000$ steps. Adding the safety constraint adds a significant challenge, as any safe policy is much more limited in exploration. In addition to the safety constraint, we add a cost for taking actions, as done in \citet{curi2020efficient}. Combining all these factors together, makes a challenging exploration task, as we show in our experiments. Further implementation details can be found in our open-source code.

\newpage
\section{Additional Experiments}
\label{sec:additional-experiments}
\paragraph{Experiments with the \textsc{Doggo} Robot}
In this experiment we compare \algo{} with the same baseline algorithms from \Cref{sec:experiments} on \sgym{}'s \textsc{Doggo} robot. We omit our results in the \textsc{PushBox} environment as all baselines failed to solve it.
\begin{figure}
    \centering
\includegraphics[clip,width=0.88\textwidth]{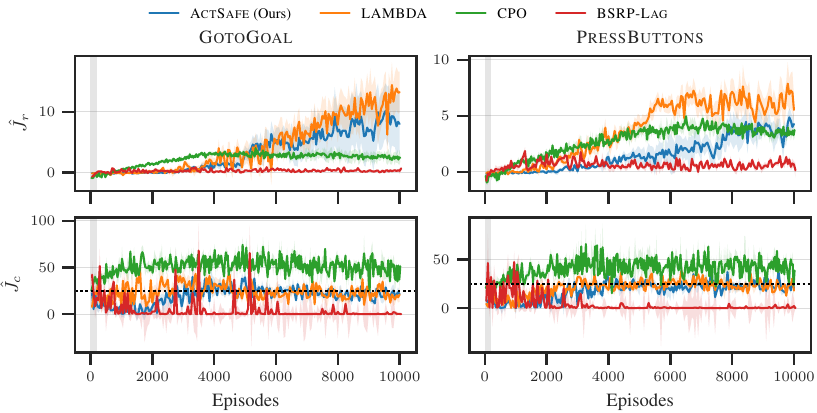}
    \caption{Performance and safety in with the \textsc{Doggo} robot.}
    \label{fig:doggo}
\end{figure}
As shown in \Cref{fig:doggo}, similarly to the results in \Cref{sec:experiments}, \algo{} maintains safety during learning, while moderately underperforming \textsc{LAMBDA}. Overall, \algo{} outperforms \textsc{CPO} both in terms of safety and performance and \textsc{BSRP-Lag} of \citet{huang2024safedreamer} in terms of performance.
\paragraph{Ablating LBSGD}
One assumptions of LBSGD that we cannot formally satisfy relates to unbiasedness of the evaluation of the objective, constraints and their gradients. In principle, satisfying this assumption will allow us to guarantee that all iterates of \Cref{eq:exploration-op-optimistic-practical} are feasible, i.e., satisfy the pessimistic constraint. 
This is in contrast to primal-dual methods, such as the Augmented Lagrangian of \citet{as2022constrained} that lacks any guarantees on feasibility during optimization. While it is hard to formally satisfy LBSGD's unbiasedness assumption, we empirically observe that LBSGD allows us to keep constraint satisfaction during learning. We present this result in \Cref{fig:lbsgd-ablate}.
\begin{figure}
    \centering
\includegraphics[clip,width=0.88\textwidth]{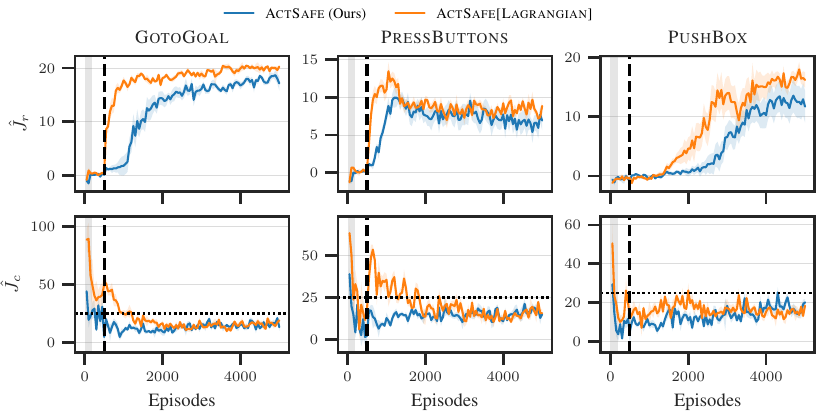}
    \caption{Augmented Lagrangian from \citet{as2022constrained} compared to LBSGD of \citet{usmanova2024log}. LBSGD significantly reduces the number of unsafe episodes.}
    \vspace{-0.5cm}
    \label{fig:lbsgd-ablate}
\end{figure}
As shown, even after initializing both variants with initial data from the burn-in period, \textsc{ActSafe[Lagrangian]} fails to satisfy the constraints throughout learning. As in the main results on safety in \Cref{fig:safety}, compared to Augmented Lagrangian, LBSGD maintains safety during learning at a slight price of performance.
\paragraph{Safe Adaptation}
Here, instead of the warm-up period of data collection, we study the effect of first training on a ``safe'' environment, like a simulator, and then continuing training on a similar environment, but with shifted dynamics. To this end, we extend \textsc{GotoGoal} from \sgym{} to two additional tasks, in which we change the motor gear and floor damping coefficients. The agent is first allowed to explore the ``sim'' environment for 300K interaction steps before being deployed on the ``real'' environment. We analyze the impact of our LBSGD optimizer and of pessimism in handling constraint violation during deployment.
As shown in \Cref{fig:adaptation}, without LBSGD and pessimism, \algo{} does not always transfer safely to the deployment environment. Furthermore, intuitively, while pessimism is crucial for maintaining safety while adapting to distribution shifts, it may sometimes hinder performance of the main objective. This experiment demonstrates that, if one has no initial data, one can use \algo{} in combination with a simulator to achieve safe exploration in practice, with a clear tradeoff of the simulator's fidelity and the degree of pessimism in \algo{}.
\begin{figure}[ht]
    \centering
\includegraphics[width=0.88\textwidth]{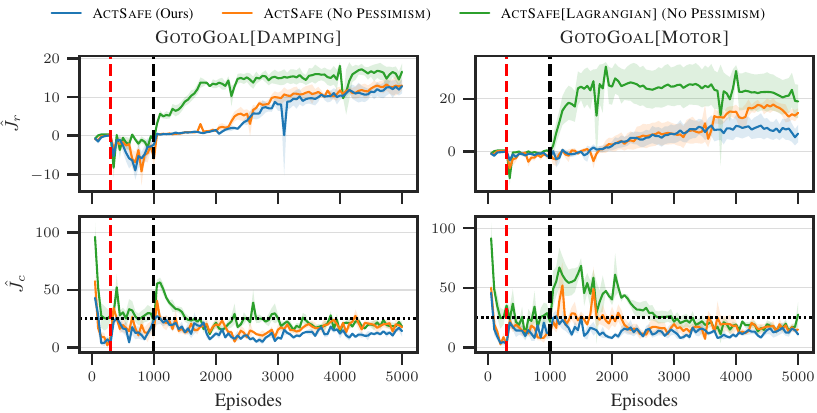}
    \caption{Adaptation to domain shifts. The red dashed vertical line represents the step after which we switch dynamics. Black dashed vertical line represents changing from active exploration to greedily maximizing the reward. We report the mean metrics across 5 seeds.}
    \vspace{-0.5cm}
    \label{fig:adaptation}
\end{figure}
\paragraph{Humanoid Proof-of-Concept}
We further demonstrate the scalability of \algo{} on the \textsc{HumanoidBench} benchmark \citep{sferrazza2024humanoidbench}. 
\begin{figure}
    \centering
    \begin{subfigure}[t]{.48\textwidth}
        \centering
        \includegraphics[width=0.7\linewidth]{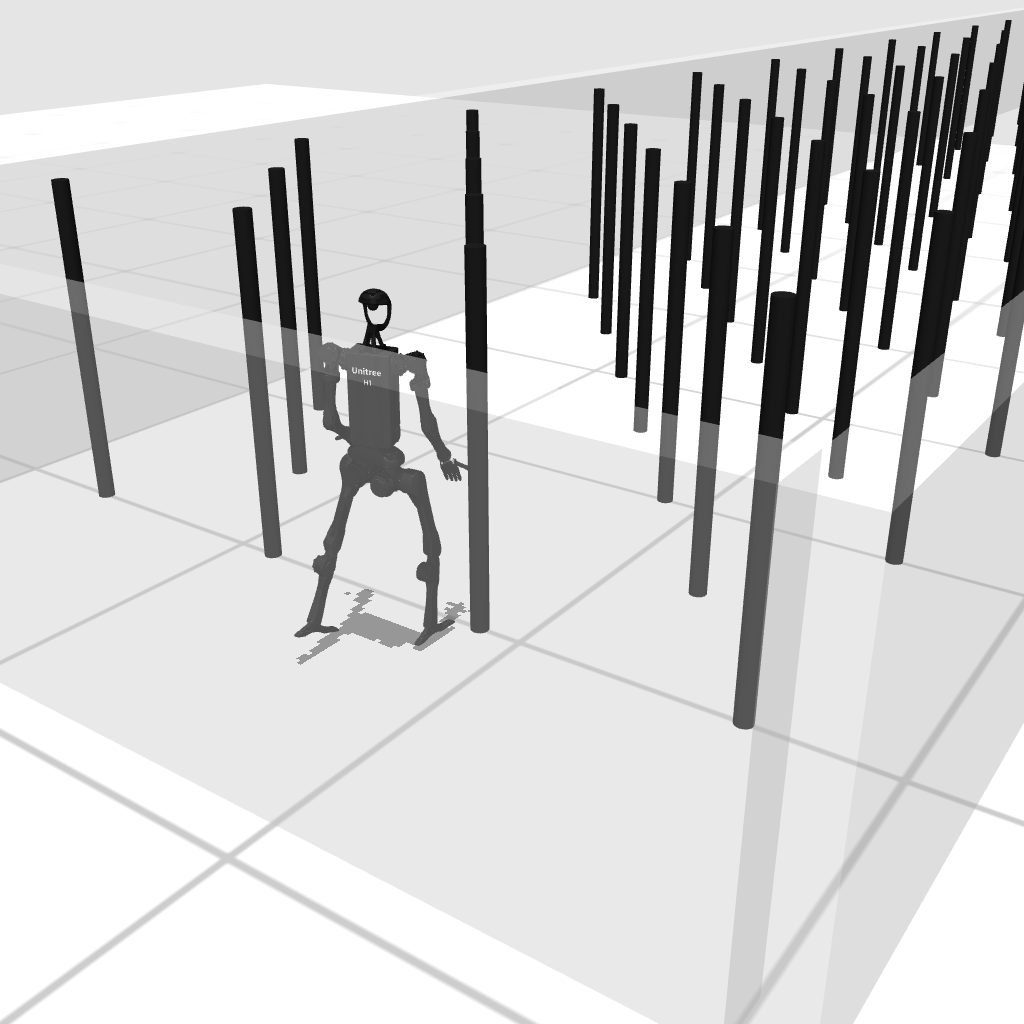}
        \caption{\textsc{Pole} task of \textsc{HumanoidBench}. The robot has to cross to the other side of the maze while avoiding hitting the poles.}
    \end{subfigure}%
    \hfill
    \begin{subfigure}[t]{.48\textwidth}
        \centering
        \includegraphics[width=\linewidth]{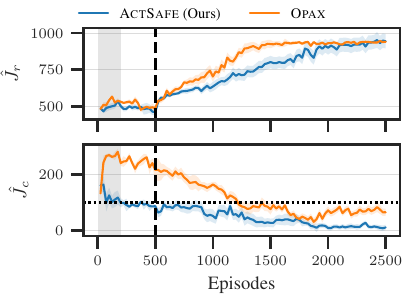}
        \caption{Performance and safety on the \textsc{Pole} task of \textsc{HumanoidBench}.}
    \end{subfigure}
    \vspace{-0.25cm}
    \caption{Overview of the Pole task and its performance metrics.}
    \label{fig:humanoid}
\end{figure}
We use a robust, low-level walking policy provided with the benchmark, and input visual observations from a third-person camera view. We compare \algo{} with \textsc{Opax} \citep{sukhija2024optimistic} on the \textsc{Pole} task, where a humanoid robot must navigate through a field of pole obstacles, as illustrated in \Cref{fig:humanoid}. In this task, the agent incurs a cost of 1 for each pole it hits and when it falls, while the reward is based on the robot's forward velocity. As shown in \Cref{fig:humanoid}, \algo{} significantly reduces the number of constraint violations compared to \textsc{Opax}, while maintaining competitive performance on the objective.

\paragraph{Comparison with \textsc{Opax} on \textsc{Cartpole}}
\begin{wrapfigure}[16]{r}
{0.5\textwidth}
\vspace{-1.25\baselineskip}
  \begin{center}
        \includegraphics{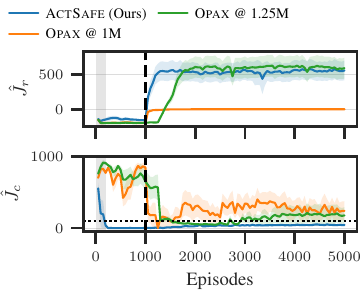}
  \end{center}
  \caption{Comparison of \algo{} and \textsc{Opax} in the \textsc{CartpoleSwingupSparse} task of \rwrl{}.}
  \label{fig:cartpole-3e-opax}
\end{wrapfigure}
We compare \algo{} with \textsc{Opax} \citep{sukhija2022scalable} on the \textsc{CartpoleSwingupSparse} task from \Cref{sec:vision-control-exp}. Both \algo{} and \textsc{Opax} rely on intrinsic rewards for exploration and model learning, however, \algo{} only considers policies from within the pessimistic safe set. We compare \algo{} with \textsc{Opax} trained for 1M and 1.25M steps of pure exploration. \algo{} uses 1M exploration steps, as in \Cref{sec:vision-control-exp}. As shown in \Cref{fig:cartpole-3e-opax}, \textsc{Opax} fails to sufficiently explore the dynamics within 1M steps. The reason being that \algo{} can explore in a much more confined state-action space, and therefore visits states with non-zero rewards quicker. This is in contrast to \textsc{Opax} which is permitted to explore unsafe action-states as well, and therefore less likely to visit these states within the given training budget. We note that a result in a similar spirit has been observed by \citet[][Figure 2]{widmer2023tuning}. While 1M steps are not enough for \textsc{Opax} to fully learn the dynamics when no constraints are imposed on the policy, in \Cref{fig:cartpole-3e-opax} we show that after having explored the dynamics for 1.25M steps, \textsc{Opax} is able to recover an optimal policy. Unsurprisingly, in both experiments \textsc{Opax} fails to satisfy the constraints, as it optimizes only for the intrinsic reward.

\end{document}